\documentclass{article}

     \PassOptionsToPackage{numbers, compress}{natbib}

\usepackage[preprint]{neurips_2019}

\usepackage[utf8]{inputenc} 
\usepackage[T1]{fontenc}    
\usepackage[colorlinks=true,linkcolor=blue,citecolor=blue]{hyperref} 
\usepackage{url}            
\usepackage{booktabs}       
\usepackage{amsfonts}       
\usepackage{nicefrac}       
\usepackage{microtype}      

\usepackage{amsfonts} 
\usepackage{amsmath}  
\usepackage{amsthm}
\usepackage{cases}    
\usepackage{wrapfig}
\usepackage{caption, graphicx}
\usepackage{multirow}

\usepackage{mathrsfs}    
\usepackage{ dsfont }                  
\usepackage{mathtools}                   
\usepackage{booktabs}
\usepackage{tikz}
\usepackage[all]{xy}
\usepackage[utf8]{inputenc}
\usepackage{mathtools}                     

\usepackage{faktor}
\usepackage{ marvosym }
\usepackage{bbm}                         
\usepackage{framed}
\usepackage{graphicx}   

\newif\ifcomments
\commentstrue

\definecolor{mycolor}{RGB}{20, 20, 122}

\usepackage{theoremref}
\numberwithin{equation}{section}

\newtheorem{theorem}{Theorem}
\newtheorem{lemma}[theorem]{Lemma}

\def\R{{\mathbb R}}

\let\on=\operatorname

\newcommand{\ud}{\,\mathrm{d}}

\newcommand{\wrt}{{w.r.t.}}
\newcommand{\ie}{{i.e.}}

\newcommand{\eg}{{e.g.}}
\newcommand{\Eg}{{E.g.}}

\newcommand{\zy}[1]{{\color{black}{#1}}}

\title{Region-specific Diffeomorphic Metric Mapping}

\author{
  Zhengyang Shen \\
  UNC Chapel Hill\\
  {\tt\small zyshen@cs.unc.edu} 
  \And 
   François-Xavier Vialard \\
   LIGM, UPEM \\
   {\tt\small francois-xavier.vialard@u-pem.fr}
   \And 
   Marc Niethammer \\
   UNC Chapel Hill \\
   {\tt\small mn@cs.unc.edu}
}

\begin{document}

\maketitle

\begin{abstract}

 We introduce a region-specific diffeomorphic metric mapping (RDMM) registration approach. RDMM is non-parametric, estimating spatio-temporal velocity fields which parameterize the sought-for spatial transformation. Regularization of these velocity fields is necessary. In contrast to existing non-parametric registration approaches using a fixed spatially-invariant regularization, for example, the large displacement diffeomorphic metric mapping (LDDMM) model, our approach allows for spatially-varying regularization which is advected via the estimated spatio-temporal velocity field. Hence, not only can our model capture large displacements, it does so with a spatio-temporal regularizer that keeps track of how regions deform, which is a more natural mathematical formulation. We explore a family of RDMM registration approaches: 1) a registration model where regions with separate regularizations are pre-defined (\eg, in an atlas space or for distinct foreground and background regions), 2) a registration model where a general spatially-varying regularizer is estimated, and 3) a registration model where the spatially-varying regularizer is obtained via an end-to-end trained deep learning (DL) model. We provide a variational derivation of RDMM, showing that the model can assure diffeomorphic transformations in the continuum, and that LDDMM is a particular instance of RDMM. To evaluate RDMM performance we experiment 1) on synthetic 2D data and 2) on two 3D datasets: knee magnetic resonance images (MRIs) of the Osteoarthritis Initiative (OAI) and computed tomography images (CT) of the lung. Results show that our framework achieves comparable performance to state-of-the-art image registration approaches, while providing additional information via a learned spatio-temporal regularizer. Further, our deep learning approach allows for very fast RDMM and LDDMM estimations. Code is available at \url{https://github.com/uncbiag/registration}.

\end{abstract}

\section{Introduction}

Quantitative analysis of medical images frequently requires the estimation of spatial correspondences, \ie image registration. For example, one may be interested in capturing knee cartilage changes over time, localized changes of brain structures, or how organs at risk move between planning and treatment for radiation treatment. Specifically, image registration seeks to estimate the spatial transformation between a source image and a target image, subject to a chosen transformation model. 



Transformations can be parameterized via low-dimensional parametric models (\eg, an affine transformation), but more flexible models are required to capture subtle local deformations. Such registration models~\cite{bajcsy1989multiresolution,shen2002hammer} may have large numbers of parameters, \eg, a large number of B-spline control points~\cite{rueckert1999nonrigid} or may even be non-parametric where vector fields are estimated~\cite{beg2005computing,modersitzki2004numerical}. Spatial regularity can be achieved by appropriate constraints on displacement fields~\cite{haber2007image} or by parameterizing the transformation via integration of a sufficiently regular stationary or time-dependent velocity field~\cite{beg2005computing,hart2009optimal,vercauteren2009diffeomorphic,chen2013large,wulff2015efficient}. Given sufficient regularization, diffeomorphic transformations can be assured in the continuum. A popular approach based on time-dependent velocity fields is LDDMM~\cite{beg2005computing,hart2009optimal}. Optimal LDDMM solutions are geodesics and minimize a geodesic distance. Consequentially, one may directly optimize over a geodesic's initial conditions in a shooting approach~\cite{vialard2012diffeomorphic}.

Most existing non-parametric image registration approaches use spatially-invariant regularizers. However, this may not be realistic. \Eg, when registering inhale to exhale images of a lung one expects large deformations of the lung, but not of the surrounding tissue. Hence, a spatially-varying regularization would be more appropriate. As the regularizer encodes the deformation model this then allows anticipating different levels of deformation at different image locations.

\begin{figure}[!ht]
\includegraphics[width=1\textwidth]{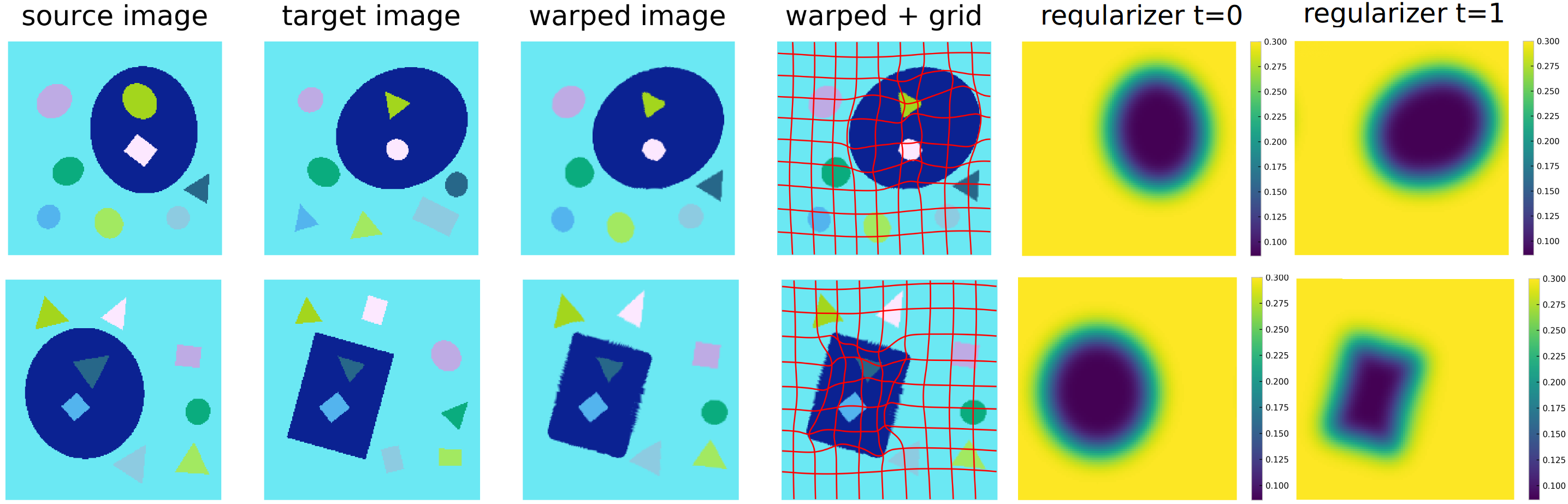}
\caption{RDMM registration example. The goal is to register the dark blue area with high fidelity (\ie, allowing large local deformations), while assuring small deformations in the cyan area.  Initially ($t=0$), a spatially-varying regularizer (fifth column) is specified in the source image space, where dark blue indicates small regularization and yellow large regularization. Specifically, regularizer values indicate effective local standard deviations of a local multi-Gaussian regularizer. Since the transformation map and the regularizer are both advected according to the estimated velocity field, the shape of the regularizer follows the shape of the deforming dark blue region and is of the same shape as the region of interest in the target space at the final time ($t=1$) (as can be seen in the second and the last columns). Furthermore, objects inside the dark blue region are indeed aligned well, whereas objects in the cyan region were not strongly deformed due to the larger regularization there.}
\label{fig:syn_2d}
\end{figure}

While spatially-varying regularizers may be used in LDDMM variants~\cite{schmah2013left} existing approaches do not allow for \emph{time-varying} spatially-varying regularizers. However, such regularizers would be natural for large displacement as they can move with a deforming image. Hence, we propose a family of registration approaches with \emph{spatio-temporal regularizers} based on advecting spatially-varying regularizers via an estimated spatio-temporal velocity field. Specifically, we extend LDDMM theory, where the original LDDMM model becomes a special case. In doing so, our entire model, including the spatio-temporal regularizer is expressed via the initial conditions of a partial differential equation. We propose three different approaches based on this model: 1) A model for which the regularizer is specified region-by-region. This would, for example, be natural when registering a labeled atlas image to a target image, as illustrated in Fig.~\ref{fig:syn_2d}.  2) A model in which the regularizer is estimated jointly with the spatio-temporal velocity fields. 3) A deep learning model which predicts the regularizer and the initial velocity field, thereby resulting in a very fast registration approach. 

\zy{{\bf Related Work}  Only limited work \cite{pace2013locally,risser2013piecewise,stefanescu2004grid,vialard2014spatially} on spatially-varying regularizers exists. The most related work~\cite{niethammer2019_cvpr} learns a spatially-varying regularizer for a stationary velocity field registration model. In~\cite{niethammer2019_cvpr} both the velocity field and the regularizer are assumed to be constant in time. In contrast, both are time-dependent in RDMM. Spatially-varying regularization has also been addressed from a Bayesian view in~\cite{simpson2015probabilistic} by putting priors on B-spline transformation parameters. However, metric estimation in~\cite{simpson2015probabilistic} is in a fixed atlas-space, whereas RDMM addresses general pairwise image registration.}\\


{\bf Contributions}: 1) We propose RDMM, a new registration model for large diffeomorphic deformations with a spatio-temporal regularizer capable of following deforming objects and thereby providing a more natural representation of deformations than existing non-parametric models, such as LDDMM. 2) Via a variational formulation we derive shooting equations that allow specifying RDMM solutions entirely based on their initial conditions: an initial momentum field and an initial spatially-varying regularizer. 3) We prove that diffeomorphisms can be obtained for RDMM in the continuum for sufficiently regular regularizers. 4) We explore an entire new family of registration models based on RDMM and provide optimization-based \emph{and} very fast deep-learning-based approaches to estimate the initial conditions of these registration models. 5) We demonstrate the utility of our approach via experiments on synthetic data and on two 3D medical image datasets.

\section{Standard LDDMM Model}
\label{sec:standard_lddmm}

LDDMM~\citep{beg2005computing} is a non-parametric registration approach based on principles from fluid mechanics. It is based on the estimation of a spatio-temporal velocity field $v(t,x)$ from which the sought-for spatial transformation $\varphi$ can be computed via integration of $\partial_t \varphi (t,x) = v(t,\varphi (t,x))\,.$  For appropriately regularized velocity fields~\citep{dupuis1998variational}, diffeomorphic transformations can be guaranteed. The optimization problem underlying LDDMM for images can be written as ($\nabla$ being the gradient; \zy{$\langle \cdot,\cdot\rangle$ indicating the inner product})
\begin{equation}
v^* = \underset{v}{\text{argmin}}~ \frac 12 \int_0^1 \|v(t) \|^2_L \ud t + \on{Sim}(I(1),I_1), \quad\text{s.t.}\quad
\partial_t I + \langle \nabla I, v\rangle=0; I(0)=I_0\enspace.
\label{eq:lddmm}
\end{equation}
Here, the goal is to register the source image $I_0$ to the target image $I_1$ in unit time. $\on{Sim}(A,B)$ is a similarity measure between images, often sum of squared differences, normalized cross correlation, or mutual information. Furthermore, we note that $I(1,y)=I_0\circ\varphi^{-1}(1,y)$, where $\varphi^{-1}$ denotes the inverse of $\varphi$ in the target image space. The evolution of this map can be expressed as \begin{equation}\label{EqFlowEquation}\partial_t \varphi^{-1} + D\varphi^{-1}v=0\,,
\end{equation} 
where $D$ denotes the Jacobian. Equivalently, in Eq.~\eqref{eq:lddmm}, we directly advect the image~\citep{hart2009optimal,vialard2012diffeomorphic} via $\partial_t I + \langle \nabla I, v\rangle=0$. To assure smooth transformations, LDDMM penalizes non-smooth velocity fields via the norm $\|v\|_L^2=\langle Lv,v\rangle$, where $L$ is a differential operator.

At optimality of Eq.~\eqref{eq:lddmm} the Euler-Lagrange equations are (div denoting the divergence)
\begin{equation}
\partial_t I + \langle \nabla I,  v\rangle = 0,~I(0)=I_0;~ 
\partial_t \lambda + \on{div}(\lambda v) = 0,~\lambda(1) = \frac{\delta}{\delta I(1)}\on{Sim}(I(1),I_1) = 0;~v=L^{-1} (\lambda\nabla I).
\label{system_lddmm}
\end{equation}
Here, $\lambda$ is the adjoint variable to $I$, also known as the scalar momentum~\citep{hart2009optimal,vialard2012diffeomorphic}. As $L^{-1}$ is a smoother it is often chosen as a convolution, \ie, $v=K\star (\lambda\nabla I)$. Note that $m(t,x) := \lambda \nabla I$ is the vector-valued momentum and thus $v=K \star m$. One can directly optimize over $v$ solving Eq.~\eqref{system_lddmm}~\citep{beg2005computing} or regard Eq.~\eqref{system_lddmm} as a constraint defining a geodesic path~\citep{vialard2012diffeomorphic} and optimize over all such solutions subject to a penalty on the initial scalar momentum as well as the similarity measure. Alternatively, one can express \zy{\footnote{Thm.~\eqref{thm:rdmm_momentum_optimality} in suppl.~\ref{sec:optimality_conditions}  provides a more generalized derivation.}} these equations entirely with respect to the vector-valued momentum, $m$, resulting in the Euler-Poincar\'e equation for diffeomorphisms (EPDiff)~\citep{younes2009evolutions}:
\begin{equation}
\partial_t m + \on{div}(v) m + D v^T(m) + D m(v) = 0,~m(0)=m_0, v=K\star m\,,
\label{system_lddmm_epdiff}
\end{equation}
which defines the evolution of the spatio-temporal velocity field based on the initial condition, $m_0$, of the momentum, from which the transformation $\varphi$ can be computed using Eq.~\eqref{EqFlowEquation}. 
Both \eqref{system_lddmm} and \eqref{system_lddmm_epdiff} can be used to implement shooting-based LDDMM~\citep{vialard2012diffeomorphic,singh2013vector}. As LDDMM preserves the momentum, $\|v\|_L^2$ is constant over time and hence a {\it shooting formulation} can be written as
\begin{equation}
    m(0)^* = \underset{m(0)}{\text{argmin}}~ \frac 12 \|v(0) \|^2_L + \on{Sim}(I(1),I_1),
    \label{lddmm_shooting}
\end{equation}
subject to the EPDiff equation~\eqref{system_lddmm_epdiff} including the advection of $\varphi^{-1}$, where $I(1)=I_0\circ\varphi^{-1}(1)$.

A shooting-formulation has multiple benefits: 1) it allows for a compact representation of $\varphi$ via its initial conditions; 2) as the optimization is \wrt~the initial conditions, a solution is a geodesic by construction; 3) these initial conditions can be predicted via deep-learning, resulting in very fast registration algorithms which inherit the theoretical properties of LDDMM~\citep{yang2017quicksilver,yang2016fast}. We therefore use this formulation as the starting point for RDMM in Sec~\ref{sec:rdmm}.



\section{Region-Specific Diffeomorphic Metric Mapping (RDMM)}
\label{sec:rdmm}

In standard LDDMM approaches, the regularizer $L$ (equivalently the kernel $K$) is spatially invariant. While recent work introduced spatially-varying metrics in LDDMM, for stationary velocity fields, or for displacement-regularized registration~\citep{risser2013piecewise,schmah2013left,niethammer2019_cvpr,pace2013locally}, all of these approaches use a temporally fixed regularizer. Hereafter, we generalize LDDMM by \emph{advecting} a spatially-varying  regularizer via the estimated spatio-temporal velocity field. Standard LDDMM is a special case of our model. 

Following~\citep{niethammer2019_cvpr}, we introduce $(V_i)_{i = 0,\ldots,N-1}$ a finite family of reproducing kernel Hilbert spaces (RKHS) which are defined by the pre-defined Gaussian kernels $K_{\sigma_i}$ with $\sigma_0 < \ldots < \sigma_{N-1}$. We use a partition of unity $w_i(x,t),{i=0,\ldots,N-1}$, on the image domain. As we want the kernel $K_{\sigma_i}$ to be active on the region determined by $w_i$ we introduce the vector field $v = \sum_{i=0}^{N-1} w_i \nu_i$ for $\nu_i \in V_i$. On this new space of vector fields, there exists a natural RKHS structure defined by
\begin{equation}
\| v \|^2_L := \inf \left\{ \sum_{i=0}^{N-1} \|\nu_i \|^2_{V_i}\,| \, v = \sum_{i=0}^{N-1} w_i \nu_i \right\}\,,
\end{equation}
whose kernel is $K = \sum_{i=0}^{N-1} w_i K_{\sigma_i} w_i$. 
 Thus the velocity reads (\zy{see suppl.~\ref{sec:kernel} for the derivation})
\begin{equation}
v=K \star m \stackrel{\mathrm{def.}}{=} \sum_{i=0}^{N-1} w_{i} K_{\sigma_i} \star (w_{i}m), w_{i} \geq 0\, ,
\label{eq:velocity_relationship}
\end{equation}
 which can capture multi-scale aspects of deformation~\cite{risser2010simultaneous}. In LDDMM, the weights are constant and pre-defined. Here, we allow spatially-dependent weights $w_i(x)$. In particular (see formulation below), we advect them via, $v(t,x)$ thereby making them spatio-temporal, \ie, $w_i(t,x)$. In this setup, weights only need to be specified at initial time $t=0$. As the Gaussian kernels are fixed convolutions can still be efficiently computed in the Fourier domain.

We prove (see suppl.~\ref{sec:mathematical_properties} for the proof) that, for sufficiently smooth weights $w_i$, the velocity field is bounded and its flow is a diffeomorphism. Following~\citep{niethammer2019_cvpr}, to assure the smoothness of the initial weights we instead optimize over \emph{initial pre-weights}, $h_{i}(0,x)\geq 0$, s.t. $w_{i}(0,x)=G_\sigma \star h_{i}(0,x)$, where $G_\sigma$ is a fixed Gaussian with a small standard deviation, $\sigma$. In addition, we constrain $\sum_{i=0}^{N-1} h_i^2(0,x)$ to locally sum to one. The optimization problem for our RDMM model then becomes
\begin{equation}
v^*,\{h_i^*\} =~\underset{v,\{h_i\}}{\text{argmin}}~ \frac 12 \int_0^1 \|v(t) \|^2_L \ud t + \on{Sim}(I(1),I_1) + \on{Reg}(\{h_i(0)\})\, ,
\label{eq:rdmm_opt_problem}
\end{equation}
subject to the constraints
\begin{gather}
\partial_t I + \langle \nabla I , v \rangle = 0,~I(0)=I_0;~
\partial_t h_i + \langle \nabla h_i , v \rangle = 0,~h_i(0)=(h_i)_0;\\ \notag
\nu_i = K_{\sigma_i} \star (w_i m); v = \sum_{i =0}^{N-1} w_i \nu_i;~
w_i = G_\sigma\star h_i\,.\label{eq:rdmm_image_constraints}
\end{gather}
As for LDDMM, we can compute the optimality conditions for Eq.~\eqref{eq:rdmm_opt_problem} which we use for shooting.

\begin{theorem}[Image-based RDMM optimality conditions] \label{thm:rdmm_image_optimality}
With the adjoints $\gamma_i$ (for $h_i$) and $\lambda$ (for $I$) and the momentum $m:= \lambda\nabla I +\sum_{i=0}^{N-1} \gamma_i \nabla h_i $ the optimality conditions for~\eqref{eq:rdmm_opt_problem} are: 
\begin{gather}
\partial_t I + \langle \nabla I , v \rangle = 0,~I(0)=I_0;~\partial_t \lambda + \on{div}(\lambda v) = 0,~-\lambda(1) +  \frac{\delta}{\delta I(1)}\on{Sim}(I(1),I_1) = 0;\\
\partial_t h_i + \langle \nabla h_i , v \rangle = 0,~h_i(0)=(h_i)_0;\\
\partial_t \gamma_i + \on{div}(\gamma_i v) = G_\sigma \star (m \cdot \nu_i),~\gamma_i(0) +\frac{\delta}{\delta h_i(0)}\on{Reg}(\{h_i(0)\}) = 0\,.
\end{gather}
subject to  
\begin{equation}
\nu_i = K_{\sigma_i} \star (w_i m);~v = \sum_{i =0}^{N-1} w_i \nu_i;~w_i=G_{\sigma}\star h_i\,.
\label{eq:vel_momentum_relationship}
\end{equation}

\end{theorem}

\begin{theorem}[Momentum-based RDMM optimality conditions]\label{thm:rdmm_momentum_optimality} The RDMM optimality conditions of Thm.~\eqref{thm:rdmm_image_optimality} can be written entirely \wrt~the momentum (as defined in Thm~\eqref{thm:rdmm_image_optimality}). They are:
\begin{gather}\label{EqREPDiff}
\partial_t \varphi^{-1} + D\varphi^{-1}v=0,~\varphi^{-1}(0,x)=x\,,\\
    \partial_t m + \on{div}(v) m + D v^T(m) + D m(v) =  \sum_{i=0}^{N-1} G_\sigma \star (m \cdot \nu_i ) \nabla h_i,~m(0)=m_0\label{GEPDiff}\,,    
\end{gather}
where $h_{i}(t,x) = h_i(0,x)\circ \varphi(t,x)^{-1}$ and 
subject to the constraints of Eq.~\eqref{eq:vel_momentum_relationship} which define the relationship between the velocity and the momentum. 
\end{theorem}

For spatially constant pre-weights, we recover EPDiff from the momentum-based RDMM optimality conditions. Instead of advecting $h_i$ via $\varphi(t,x)^{-1}$, we can alternatively advect the pre-weights directly, as $\partial_t h_i + \langle \nabla h_i,v\rangle=0,~h_i(0)=(h_i)_0$. For the image-based and the momentum-based formulations, the velocity field $v$ is obtained by smoothing the momentum, $v=\sum_{i =0}^{N-1} w_i K_{\sigma_i} \star (w_i m)$. 





{\bf Regularization of the Regularizer}

\zy{Given the standard deviations $\sigma_0 < \ldots < \sigma_{N-1}$ of Eq.~\eqref{eq:velocity_relationship}, assigning larger weights to the Gaussians with larger standard deviation will result in smoother (i.e., more regular) and therefore simpler transformations. To encourage choosing simpler transformations we follow~\citep{niethammer2019_cvpr}, where a simple optimal mass transport (OMT) penalty on the weights is used. Such an OMT penalty is sensible as the non-negative weights, $\{w_i\}$, sum to one and can therefore be considered a discrete probability distribution. The chosen OMT penalty is designed to measure differences from the probability distribution assigning all weight to the Gaussian with the largest standard deviation, i.e., $w_{N-1}=1$ and all other weights being zero.} Specifically, the OMT penalty of~\citep{niethammer2019_cvpr} is $ \left|\log \frac{\sigma_{N-1}}{\sigma_{0}}\right|^{-s} \sum_{i=0}^{N-1} \zy{w_{i}}\left|\log \frac{\sigma_{N-1}}{\sigma_{i}}\right|^{s}$, where $s$ is a chosen power. To make the penalty more consistent with penalizing weights for a standard multi-Gaussian regularizer (as our regularizer contains effectively the weights squared) we do not penalize the weights directly, but instead penalize their squares using the same form of OMT penalty. Further, as the regularization only affects the initial conditions for the pre-weights, the evolution equations for the optimality conditions (\ie, the modified EPDiff equation) do not change. Additional regularizers, such as total variation terms as proposed in~\citep{niethammer2019_cvpr}, are possible and easy to integrate into our RDMM framework as they only affect initial conditions. For simplicity, we focus on regularizing via OMT.

{\bf Shooting Formulation}

As energy is conserved (see suppl.~\ref{sec:energy_conservation} for the proof) the momentum-based shooting formulation becomes
\begin{equation}
    m(0)^*,\{h_i(0)^*\} =~\underset{m(0),\{h_i(0)\}}{\text{argmin}}~ \frac 12 \|v(0) \|^2_L + \on{Sim}(I(1),I_1) + \on{Reg}(\{h_i(0)\})\, ,
\label{rdmm_shooting}
\end{equation}
subject to the evolution equations of Thm.~\ref{thm:rdmm_momentum_optimality}. Similarly, the shooting formulation can use the image-based evolution equations of Thm.~\ref{thm:rdmm_image_optimality} where optimization would be over $\lambda(0)$ instead of $m(0)$.



\begin{figure}[!t]
\includegraphics[width=1\textwidth]{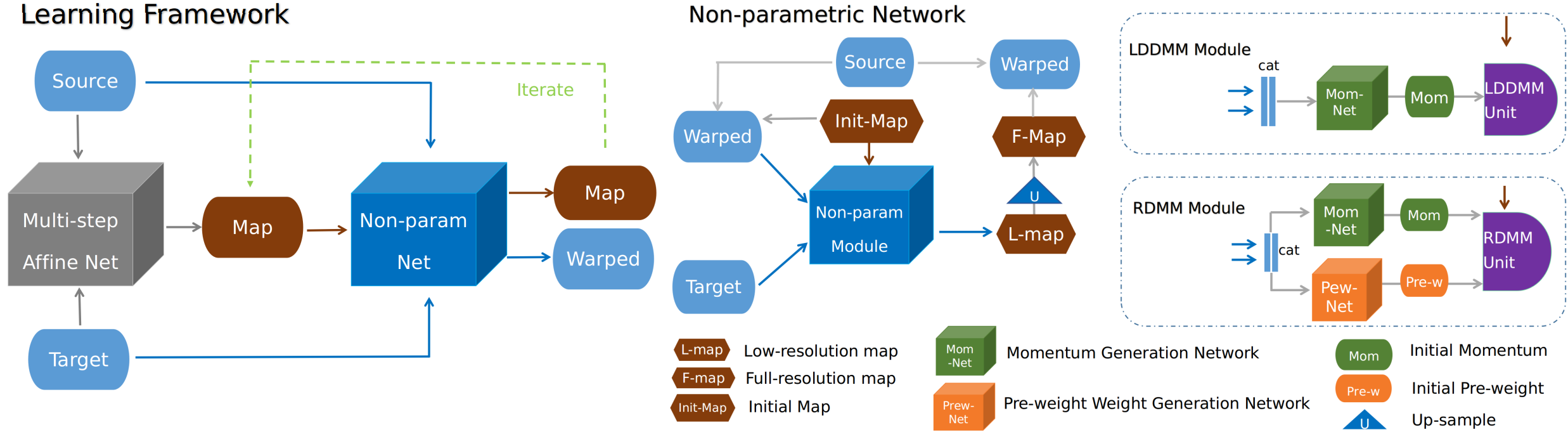}
\caption{Illustration of the learning framework (left) and its non-parametric registration component (right). A multi-step affine-network first predicts the affine transformation map~\cite{shen2019networks} followed by an iterable non-parametric registration to estimate the final transformation map. The LDDMM component uses one network to generate the initial momentum. RDMM also uses a second network to predict the initial regularizer pre-weights. We integrate the RDMM evolution equations at low-resolution (based on the predicted initial conditions) to save memory. The final transformation map is obtained via upsampling. See suppl.~\ref{sec:unit} for the detailed structure of the LDDMM/RDMM units.} 
\label{fig:framework_rdmm}
\end{figure}

\section{Learning Framework}
\label{sec:learning_framework}

The parameters for RDMM, \ie, the initial momentum and the initial pre-weights, can be obtained by numerical optimization, either over the momentum \emph{and} the pre-weights or only over the momentum if the pre-weights are prescribed. Just as for the LDDMM model, such a numerical optimization is computationally costly. Consequentially, various deep-learning (DL) approaches have been proposed to instead predict displacements~\citep{balakrishnan2018unsupervised,cao2018deformable}, stationary velocity~\citep{rohe2017svf} or momentum fields~\citep{yang2017quicksilver,niethammer2019_cvpr}.  Supervised~\cite{yang2017quicksilver,yang2016fast} and unsupervised~\cite{jaderberg2015spatial,de2017end,li2017non,balakrishnan2018unsupervised,dalca2018unsupervised} DL registration approaches exist. All of them are fast as only a regression solution needs to be evaluated at test time and no further numerical optimization is necessary. Additionally, such DL models benefit from learning over an entire population instead of relying only on information from given image-pairs. 

Most non-parametric registration approaches are not invariant to affine transformations based on the chosen regularizers. Hence, for such non-parametric methods, affine registration is typically performed first as a pre-registration step to account for large, global displacements or rotations. Similarly, we make use of a two-step learning framework (Fig.~\ref{fig:framework_rdmm} (left)) learning affine transformations and subsequent non-parametric deformations separately. For the affine part, a multi-step affine network is used to predict the affine transformation following~\citep{shen2019networks}. For the non-parametric part, we use two different deep learning approaches, illustrated in the right part of Fig.~\ref{fig:framework_rdmm}, to predict 1) the LDDMM initial momentum, $m_0$, in Eq.~\eqref{lddmm_shooting} and 2) the RDMM initial momentum, $m_0$, and the pre-weights, $\{h_i(0)^*\}$, in Eq.~\eqref{rdmm_shooting}. Overall, including the affine part, we use two networks for LDDMM prediction and three networks for RDMM prediction.


We use low-resolution maps and map compositions for the momentum and the pre-weight networks. This reduces computational cost significantly. The final transformation map is obtained via upsampling, which is reasonable as we assume smooth transformations. We use 3D UNets~\citep{cciccek2016_3d_unet} for momentum and pre-weight prediction. Both the affine and the non-parametric networks can be iterated to refine the prediction results: \ie~ the input source image and the initial map are replaced with the currently warped image and the transformation map respectively for the next iteration.

\zy{During training of the non-parametric part, the gradient is first backpropagated through the differentiable interpolation operator, then through the LDDMM/RDMM unit, followed by the momentum generation network and the pre-weight network.}

{\it Inverse Consistency}: For the DL approaches we follow~\cite{shen2019networks} and compute bidirectional (source to target denoted as $^{st}$ and target to source denoted as $^{ts}$) registration losses and an additional symmetry loss, $\|(\varphi^{s t})^{-1} \circ(\varphi^{t s})^{-1}-i d\|_{2}^{2}$, where $id$ refers to the identity map. This encourages symmetric consistency. 

\section{Experimental Results and Setup}
\label{sec:experimental_results}
 {\bf Datasets}: To demonstrate the behavior of RDMM, we evaluate the model on three datasets: 1) a synthetic dataset for illustration, 2) a 3D computed tomography dataset (CT) of a lung, and 3) a large 3D magnetic resonance imaging (MRI) dataset of the knee from the Osteoarthritis Initiative (OAI).\\
{\it The synthetic dataset} consists of three types of shapes (rectangles, triangles, and ellipses). There is one foreground object in each image with two objects inside and at most five objects outside. Each source image object  has a counterpart in the target image; the shift, scale, and rotations are random. 
We generated 40 image pairs of size $200^2$ for evaluation. Fig.~\ref{fig:syn_2d} shows example synthetic images.\\
{\it The lung dataset} 
consists of 49 inspiration/expiration image pairs with lung segmentations. Each image is of size $160^3$. We register from the expiration phase to the inspiration phase for all 49 pairs.\\
{\it The OAI dataset} 
consists of 176 manually labeled MRI from 88 patients (2 longitudinal scans per patient) and 22,950 unlabeled MR images from 2,444 patients. Labels are available for femoral and tibial cartilage. We divide the patients into \zy{training (2,800 pairs),  validation (50 pairs) and testing groups (300 pairs)}, with the same evaluation settings as for the cross-subject experiments in \cite{shen2019networks}.

{\bf Deformation models}: Affine registration is performed before each LDDMM/RDMM registration.

{\it Affine model}:
We implemented a multi-scale affine model solved via numerical optimization and a multi-step deep neural network to predict the affine transformation parameters.

{\it Family of non-parametric models}: We implemented both {\it optimization} and {\it deep-learning} versions of a family of non-parametric registration methods: a vector-momentum based stationary velocity field model (vSVF) ($v(x)$, $w=const$), LDDMM ($v(t,x)$, $w=const$), and RDMM ($v(t,x)$, $w(t,x)$). We use the dopri5 solver using the adjoint sensitivity method~\citep{chen2018neural} to integrate the evolution equations in time. For solutions based on numerical optimization, we use a multi-scale strategy with L-BGFS~\cite{liu1989limited} as the optimizer. For the deep learning models, we compute solutions for a low-resolution map (factor of 0.5) which is then upsampled. We use Adam~\cite{kingma2014adam} for optimization.

{\bf Image similarity measure}: We use multi-kernel Localized Normalized Cross Correlation (mk-LNCC) \cite{shen2019networks}. mk-LNCC computes localized normalized cross correlation (NCC) with different window sizes and combines these measures via a weighted sum.

{\bf Weight visualization}: To illustrate the behavior of the RDMM model, we visualize the estimated standard deviations, \ie the square root of the local variance $\sigma^{2}(x)=\sum_{i=0}^{N-1} w^2_{i}(x) \sigma_{i}^{2}$.

{\bf Estimation approaches}: To illustrate different aspects of our approach we perform three types of RDMM registrations: 1) {\it registration with a pre-defined regularizer} (Sec.~\ref{subsec:rdmm_predefined}), 2) {\it registration with simultaneous optimization of the regularizer}, via optimization of the initial momentum and pre-weights (Sec.~\ref{subsec:rdmm_optimized_regularizer}), and 3) {\it registration via deep learning} \emph{predicting} the initial momentum and regularizer pre-weights (Sec.~\ref{subsec:rdmm_deep_learning}). Detailed settings for all approaches are in the suppl.~\ref{sec:experimental_settings}.


\subsection {Registration with a pre-defined regularizer}
\label{subsec:rdmm_predefined}

To illustrate the base capabilities of our models, we prescribe an initial spatially-varying regularizer in the source image space. 
We show experimental results for pair-wise registration of the synthetic data as well as for the 3D lung volumes.

Fig.~\ref{fig:syn_2d} shows the registration result for an example synthetic image pair. We use small regularization in the blue area and large regularization in the surrounding area. As expected, most of the deformations occur inside the blue area as the regularizer is more permissive there. We also observe that the regularizer is indeed advected with the image. For the real lung image data we use a small regularizer inside the lung (as specified by the given lung mask) and a large regularizer in the surrounding tissue. Fig.~\ref{fig:lung} shows that most of the deformations are indeed inside the lung area while the deformation outside the lung is highly regularized as desired. We evaluate the Dice score between the warped lung and the target lung, achieving $95.22\%$ on average (over all inhalation/exhalation pairs). Fig.~\ref{fig:lung} also shows the determinant of the Jacobian of the transformation map $J_{\varphi^{-1}}(x) := |D{\varphi^{-1}}(x)|$ (defined in target space): the lung region shows small values (illustrating expansion) while other region are either volume preserved (close to 1) or compressed (bigger than 1). Overall the deformations are smooth.

\begin{figure}[!ht]
  \begin{center}
    \includegraphics[width=1.0\textwidth]{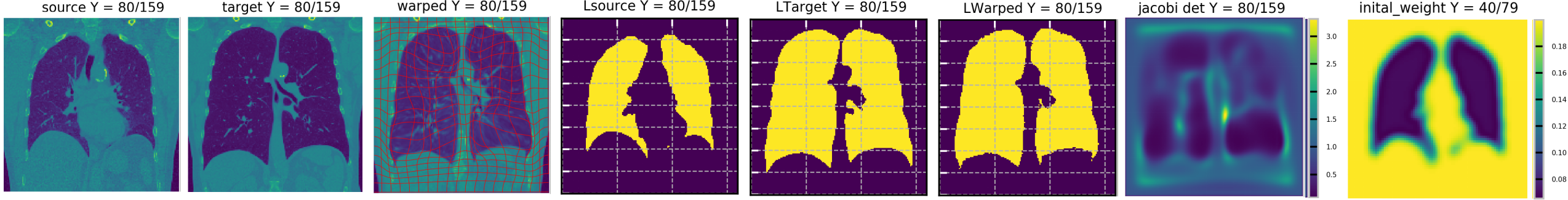}
  \end{center}
  \caption{RDMM lung registration result with a pre-defined regularizer. Lung images at expiration (source) are registered to the corresponding  inspiration (target) images. Last column: Inside the lung a regularizer with small standard deviation ($\sigma_i= \{0.04, 0.06, 0.08\}$ , $h_0^2=\{0.1, 0.4, 0.5\}$) and outside the lung with large standard deviation is used ($\sigma_i= \{0.2\}$, $h_0^2=\{1.0\}$). Deformations are largely inside the lung, the surrounding tissue is well regularized as expected. Columns 1 to 3 show results in image space while columns 4 to 6 refer to the results in label space. The second to last column shows the determinant of the Jacobian of the spatial transformation, $\varphi^{-1}$.\vspace{-0.5cm} }
  \label{fig:lung}
\end{figure}

     


\subsection{Registration with an optimized regularizer}
\label{subsec:rdmm_optimized_regularizer}

\begin{figure}[!h]
   \begin{minipage}{0.7\textwidth}
     \includegraphics[width=1\textwidth]{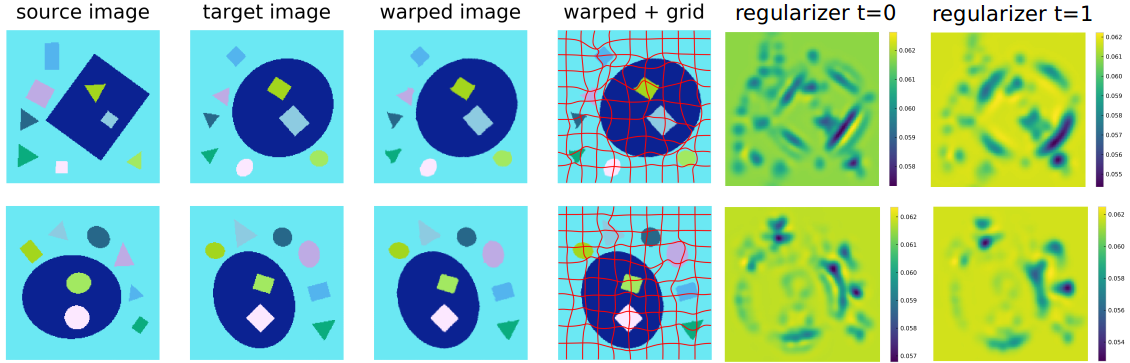}
    \caption{Illustration of the RDMM registration results with an optimized regularizer on the synthetic dataset. All objects are warped from the source image space to the target image space. The last two columns show the regularizer ($\sigma(x)$) at $t=0$ and $t=1$ respectively. }
    \label{fig:rdmm_opt_synth}
   \end{minipage}\hfill
   \begin{minipage}{0.25\textwidth}
   \vspace{-0.5em}
     \centering
     \includegraphics[width=1\linewidth]{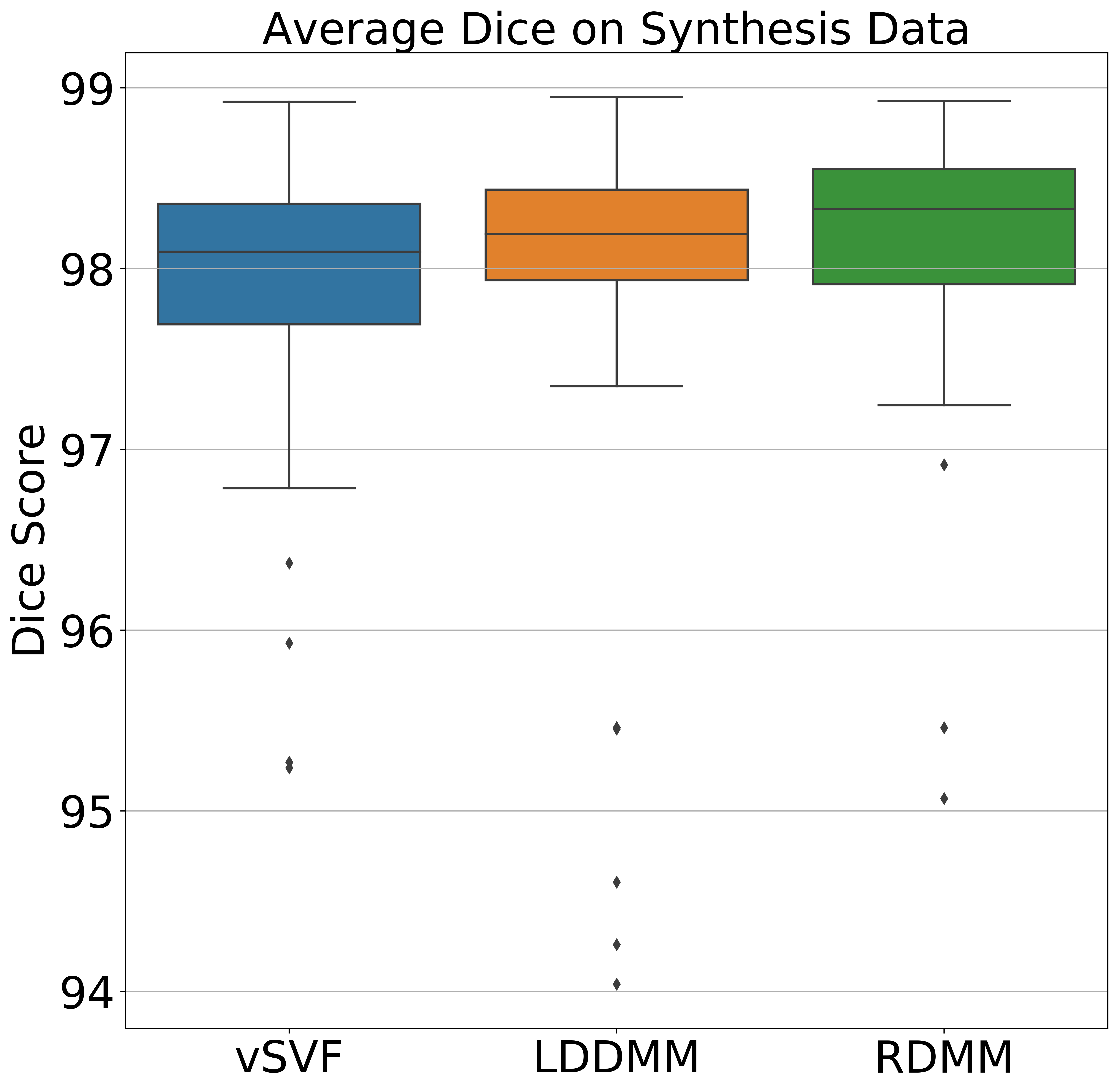}
     \caption{Average Dice scores for all objects. Left to right: SVF, LDDMM and RDMM.}\label{fig:syth_boxplot}
   \end{minipage}
\end{figure}

In contrast to the experiments in Sec.~\ref{subsec:rdmm_predefined}, we jointly estimate the initial momentum \emph{and} the initial regularizer pre-weights  for pairwise registration. We use the synthetic data and the knee MRIs.

Fig.~\ref{fig:rdmm_opt_synth} shows that the registration warped every object in the source image to the corresponding object in the target image. The visualization of the initial regularizer shows that low regularization is assigned close to object edges making them locally deformable. The visualizations of the regularizer at the initial time point ($t=0$) and the final time point ($t=1$) show that it deforms with the image. That low regularization values are localized is sensible as the OMT regularizer prefers spatially sparse solutions (in the sense of sparsely assigning low levels of regularity). If the desired deformation model is piecewise constant our RDMM model could readily be combined with a total variation penalty as in~\citep{niethammer2019_cvpr}. Fig.~\ref{fig:syth_boxplot} compares average Dice scores for all objects for vSVF, LDDMM and RDMM separately. They all achieve high and comparable performance indicating good registration quality. But only RDMM provides additional information about local regularity.

We further evaluate RDMM on 300 images pairs from the OAI dataset. The {\it optimization methods} section of Tab.~\ref{tab:per_compare} compares registration performance for different optimization-based algorithms. RDMM achieves \zy{high performance}. While RDMM is in theory diffeomorphic, we observe some foldings, whereas no such foldings appear for LDDMM and SVF. This is likely due to inaccuracies when discretizing the evolution equations and when discretizing the determinant of the Jacobian. Further, RDMM may locally exhibit stronger deformations than LDDMM or vSVF, especially when the local regularization (via OMT) is small, making it numerically more challenging. Most of the folds appear at the \zy{image boundary} (and are hence likely due to boundary discretization artifacts) or due to anatomical inconsistency in the source and target images, where large deformations may be estimated.

\begin{figure}
  
  \begin{minipage}[b]{0.425\textwidth}
    \centering
\scalebox{0.57}{
\tabcolsep=-0.04cm
\begin{tabular}{|cccc|}
  \hline
\multirow{2}{*}{Method} &
\multicolumn{2}{c}{OAI Dataset} &
\multicolumn{1}{c|}{} \\
 & Dice & Folds& Time (s) \\\hline
---------&  Affine Methods &----------&-------- \\
affine-NiftyReg    & 30.43 (12.11) &0  & 45\\
affine-opt  & 34.49 (18.07) &0  & 8\\     
affine-net & \textbf{44.58} (7.74) &0  & 0.20  \\

-----------& Optimization Methods &-------&----------\\
Demons\cite{vercauteren2009diffeomorphic,vercauteren2008symmetric}
& 63.47 (9.52) &0.56  &114\\
SyN\cite{avants2009advanced,avants2008symmetric}
& 65.71 (15.01) &0  &1330\\
NiftyReg-NMI\cite{ourselin2001reconstructing,modat2014global,rueckert1999nonrigid,modat2010fast}
& 59.65 (7.62) &0  &143\\
NiftyReg-LNCC   & 67.92 (5.24) &35.19 &270\\
vSVF-opt &67.35 (9.73)  &0  &79\\
LDDMM-opt &67.72(8.94)  & 0 & 457\\
RDMM-opt  & \textbf{68.18(8.36)} &17.37 &627\\
----------& Learning-based Methods & -------&------

---- \\
VoxelMorph\cite{dalca2018unsupervised}(with aff) &66.08 (5.13) &3.31 &0.31\\
vSVF-net \cite{shen2019networks}& 67.59 (4.47) & 0.39 &0.62\\
LDDMM-net &67.63(4.51) &0& 0.85 \\
RDMM-net &\textbf{67.94(4.40)} &0.47 &1.1\\
\hline
\end{tabular}}
\end{minipage}
\begin{minipage}[b]{0.575\textwidth}
    \centering
    \includegraphics[width=\textwidth]{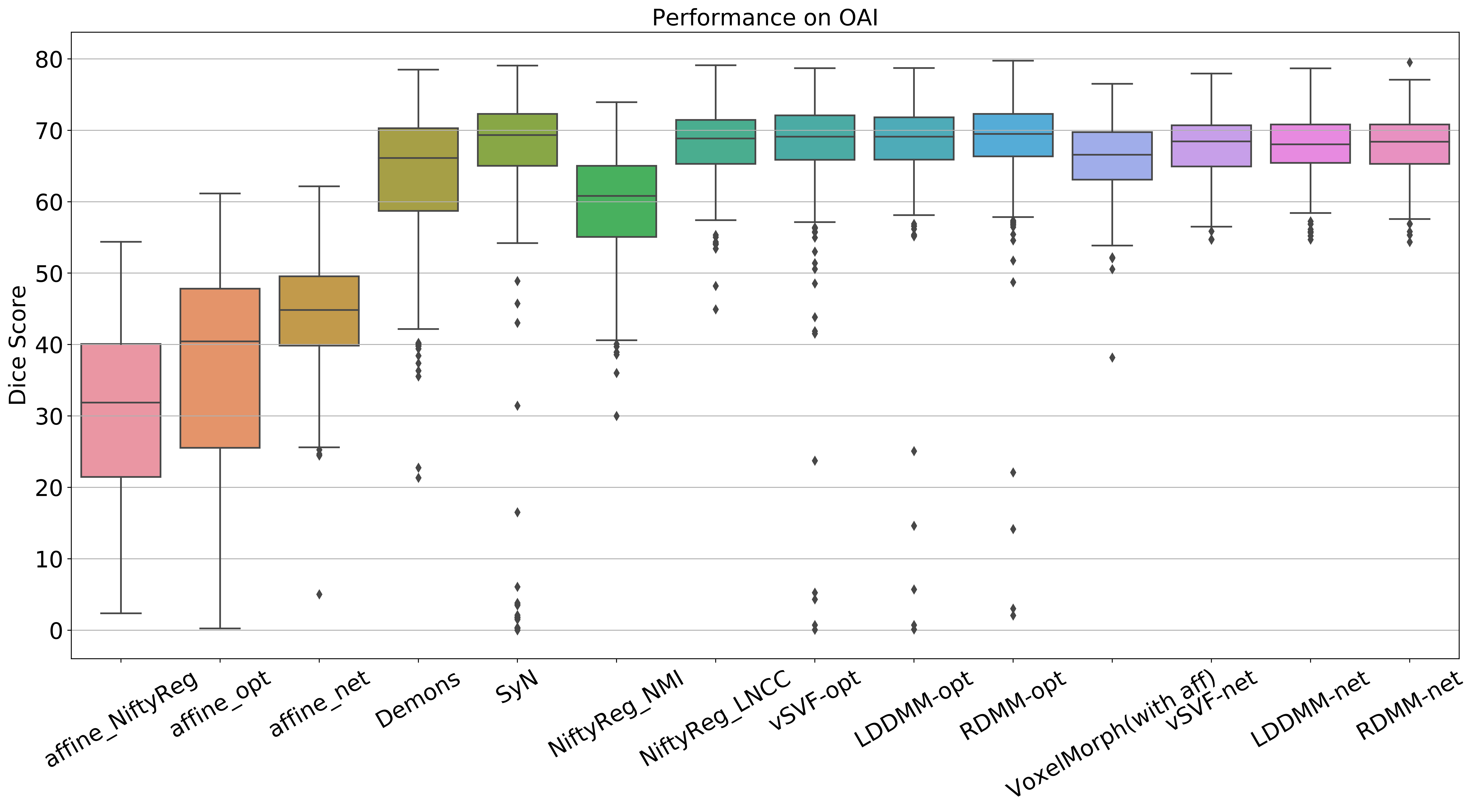}
    \par\vspace{-2.25cm}
  \end{minipage}%
\caption{\label{tab:per_compare} Comparison of registration methods for cross-subject registrations on the OAI dataset based on Dice scores. \textit{-opt} and \textit{-net} refer to optimization- and DL-based methods respectively. For all DL methods, we report performance after two-step refinement. \textit{Folds} refers to the absolute value of the sum of the determinant of the Jacobian in the folding area (\ie, where the determinant of the Jacobian is negative); \textit{Time} refers to the average registration time for a single image pair.\vspace{-0.25cm}}
\end{figure}

\begin{figure}[!t]
  \begin{center}
    \includegraphics[width=1.0\textwidth]{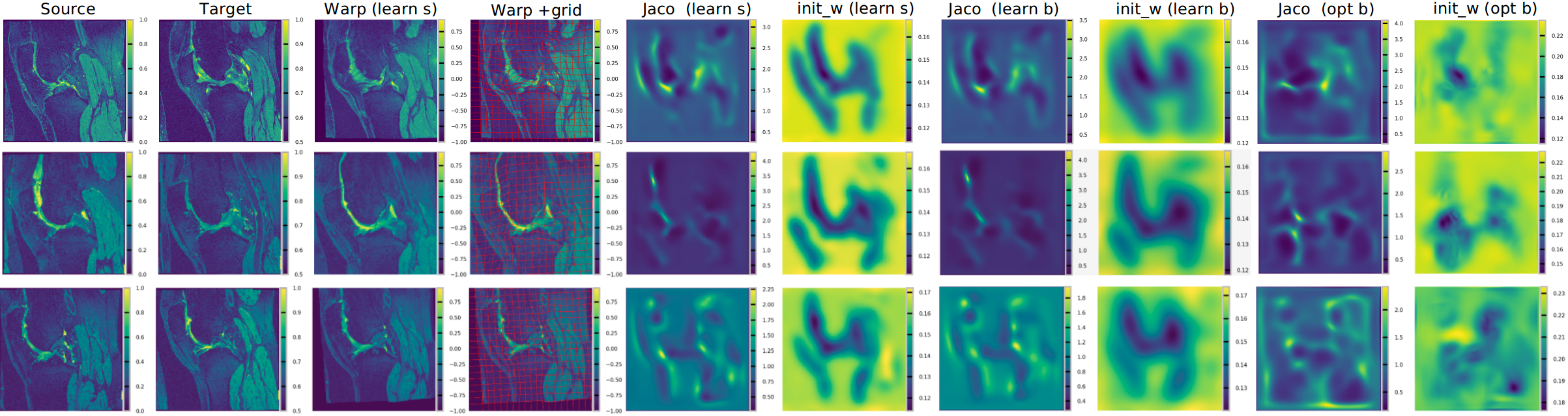}
  \end{center}
  \caption{\label{fig:oai_img}Illustration of RDMM registration results on the OAI dataset. "* s" and "* b" refer to the regularizer with $\sigma$ in $G_\sigma$ set to $0.04$ and $0.06$ respectively; "learn *" and "opt *" refer to a learnt regularizer and an optimized one respectively; "Jaco" refers to the absolute value of the determinant of the Jacobian; "init\_w" refers to the initial weight map of the regularizer (as visualized via $\sigma(x)$). The first four columns refer to registration results in image space.\vspace{-0.5cm}}
\end{figure}

\subsection{Registration with a learnt regularizer via deep learning}
\label{subsec:rdmm_deep_learning}

Finally, we evaluate our learning framework for non-parametric registration approaches on the OAI dataset. For vSVF, we follow~\citep{shen2019networks} to predict the momentum. We implement the same approach for LDDMM, where the vSVF inference unit is replaced by one for LDDMM (\ie, instead of advecting via a stationary velocity field we integrate EPDiff). Similarly, for RDMM, the inference unit is replaced by the evolution equation for RDMM and we use an additional network to predict the regularizer pre-weights. Fig.~\ref{tab:per_compare} shows the results. The non-parametric DL models achieve comparable performance to their optimization counterparts, but are much faster, while learning-based RDMM simultaneously predicts the initial regularizer which captures aspects of the knee anatomy. Fig.~\ref{fig:oai_img} shows the determinant of the Jacobian of the transformation map. It is overall smooth and folds are much less frequent than for the corresponding optimization approach, because the DL model penalizes transformation asymmetry. Fig.~\ref{fig:oai_img} also clearly illustrates the benefit of training the DL model based on a large image dataset: compared with the optimization approach (which only works on individual image pairs), the initial regularizer predicted by the deep network captures anatomically meaningful information much better: the bone (femur and tibia) and the surrounding tissue show large regularity.


\section{Conclusion and Future Work}
We introduced RDMM, a generalization of LDDMM registration which allows for spatially-varying regularizers advected with a deforming image. In RDMM, both the estimated velocity field and the estimated regularizer are time- and spatially-varying. We used a variational approach to derive shooting equations which generalize EPDiff and allow the parameterization of RDMM using only the initial momentum and regularizer pre-weights. We also prove that diffeomorphic transformation can be obtained for RDMM with sufficiently regular regularizers. Experiments with pre-defined, optimized, and learnt regularizers show that RDMM is flexible and its solutions can be estimated quickly via deep learning.\\
Future work could focus on numerical aspects and explore different initial constraints, such as total-variation constraints, depending on the desired deformation model. Indeed, a promising avenue of research consists in learning regularizers which include more physical/mechanical a-priori information in the deformation model. For instance, a possible first step in this direction consists in parameterizing non-isotropic kernels to favor deformations in particular directions. 

{\bf Acknowledgements} Research reported in this work was supported by the National Institutes of Health (NIH) and  the  National  Science  Foundation  (NSF)  under  award numbers  NSF  EECS-1711776  and  NIH  1R01AR072013. The content is solely the responsibility of the authors and does not necessarily represent the official views of the NIH or the NSF. We would also like to thank Dr. Ra\'ul San Jos\'e Est\'epar for providing the lung data.

{\small
\bibliographystyle{abbrv}

\begin{thebibliography}{10}

\bibitem{avants2008symmetric}
B.~B. Avants, C.~L. Epstein, M.~Grossman, and J.~C. Gee.
\newblock Symmetric diffeomorphic image registration with cross-correlation:
  evaluating automated labeling of elderly and neurodegenerative brain.
\newblock {\em Medical image analysis}, 12(1):26--41, 2008.

\bibitem{avants2009advanced}
B.~B. Avants, N.~Tustison, and G.~Song.
\newblock Advanced normalization tools {(ANTS)}.
\newblock {\em Insight j}, 2:1--35, 2009.

\bibitem{bajcsy1989multiresolution}
R.~Bajcsy and S.~Kova{\v{c}}i{\v{c}}.
\newblock Multiresolution elastic matching.
\newblock {\em CVGIP}, 46(1):1--21, 1989.

\bibitem{balakrishnan2018unsupervised}
G.~Balakrishnan, A.~Zhao, M.~R. Sabuncu, J.~Guttag, and A.~V. Dalca.
\newblock An unsupervised learning model for deformable medical image
  registration.
\newblock In {\em CVPR}, pages 9252--9260, 2018.

\bibitem{beg2005computing}
M.~F. Beg, M.~I. Miller, A.~Trouv{\'e}, and L.~Younes.
\newblock Computing large deformation metric mappings via geodesic flows of
  diffeomorphisms.
\newblock {\em IJCV}, 61(2):139--157, 2005.

\bibitem{cao2018deformable}
X.~Cao, J.~Yang, J.~Zhang, Q.~Wang, P.-T. Yap, and D.~Shen.
\newblock Deformable image registration using a cue-aware deep regression
  network.
\newblock {\em IEEE Transactions on Biomedical Engineering}, 65(9):1900--1911,
  2018.

\bibitem{chen2018neural}
T.~Q. Chen, Y.~Rubanova, J.~Bettencourt, and D.~K. Duvenaud.
\newblock Neural ordinary differential equations.
\newblock In {\em Advances in Neural Information Processing Systems}, pages
  6571--6583, 2018.

\bibitem{chen2013large}
Z.~Chen, H.~Jin, Z.~Lin, S.~Cohen, and Y.~Wu.
\newblock Large displacement optical flow from nearest neighbor fields.
\newblock In {\em CVPR}, pages 2443--2450, 2013.

\bibitem{cciccek2016_3d_unet}
{\"O}.~{\c{C}}i{\c{c}}ek, A.~Abdulkadir, S.~S. Lienkamp, T.~Brox, and
  O.~Ronneberger.
\newblock {3D U-Net}: learning dense volumetric segmentation from sparse
  annotation.
\newblock In {\em International conference on medical image computing and
  computer-assisted intervention}, pages 424--432. Springer, 2016.

\bibitem{dalca2018unsupervised}
A.~V. Dalca, G.~Balakrishnan, J.~Guttag, and M.~R. Sabuncu.
\newblock Unsupervised learning for fast probabilistic diffeomorphic
  registration.
\newblock {\em MICCAI}, 2018.

\bibitem{de2017end}
B.~D. de~Vos, F.~F. Berendsen, M.~A. Viergever, M.~Staring, and I.~I{\v{s}}gum.
\newblock End-to-end unsupervised deformable image registration with a
  convolutional neural network.
\newblock In {\em DLMIA}, pages 204--212. Springer, 2017.

\bibitem{delon2012}
J.~Delon, J.~Salomon, and A.~Sobolevski.
\newblock Local matching indicators for transport problems with concave costs.
\newblock {\em SIAM Journal on Discrete Mathematics}, 26(2):801--827, 2012.

\bibitem{dupuis1998variational}
P.~Dupuis, U.~Grenander, and M.~I. Miller.
\newblock Variational problems on flows of diffeomorphisms for image matching.
\newblock {\em Quarterly of applied mathematics}, pages 587--600, 1998.

\bibitem{haber2007image}
E.~Haber and J.~Modersitzki.
\newblock Image registration with guaranteed displacement regularity.
\newblock {\em International Journal of Computer Vision}, 71(3):361--372, 2007.

\bibitem{hart2009optimal}
G.~L. Hart, C.~Zach, and M.~Niethammer.
\newblock An optimal control approach for deformable registration.
\newblock In {\em CVPR}, pages 9--16. IEEE, 2009.

\bibitem{jaderberg2015spatial}
M.~Jaderberg, K.~Simonyan, A.~Zisserman, et~al.
\newblock Spatial transformer networks.
\newblock In {\em NIPS}, pages 2017--2025, 2015.

\bibitem{kingma2014adam}
D.~P. Kingma and J.~Ba.
\newblock Adam: A method for stochastic optimization.
\newblock {\em ICLR}, 2014.

\bibitem{li2017non}
H.~Li and Y.~Fan.
\newblock Non-rigid image registration using fully convolutional networks with
  deep self-supervision.
\newblock {\em ISBI}, 2018.

\bibitem{liu1989limited}
D.~C. Liu and J.~Nocedal.
\newblock On the limited memory {BFGS} method for large scale optimization.
\newblock {\em Mathematical programming}, 45(1-3):503--528, 1989.

\bibitem{modat2014global}
M.~Modat, D.~M. Cash, P.~Daga, G.~P. Winston, J.~S. Duncan, and S.~Ourselin.
\newblock Global image registration using a symmetric block-matching approach.
\newblock {\em Journal of Medical Imaging}, 1(2):024003, 2014.

\bibitem{modat2010fast}
M.~Modat, G.~R. Ridgway, Z.~A. Taylor, M.~Lehmann, J.~Barnes, D.~J. Hawkes,
  N.~C. Fox, and S.~Ourselin.
\newblock Fast free-form deformation using graphics processing units.
\newblock {\em Computer methods and programs in biomedicine}, 98(3):278--284,
  2010.

\bibitem{modersitzki2004numerical}
J.~Modersitzki.
\newblock {\em Numerical methods for image registration}.
\newblock Oxford University Press on Demand, 2004.

\bibitem{niethammer2019_cvpr}
M.~Niethammer, R.~Kwitt, and F.-X. Vialard.
\newblock Metric learning for image registration.
\newblock {\em CVPR}, 2019.

\bibitem{nocedal2006numerical}
J.~Nocedal and S.~Wright.
\newblock {\em Numerical optimization}.
\newblock Springer Science \& Business Media, 2006.

\bibitem{ourselin2001reconstructing}
S.~Ourselin, A.~Roche, G.~Subsol, X.~Pennec, and N.~Ayache.
\newblock Reconstructing a {3D} structure from serial histological sections.
\newblock {\em Image and vision computing}, 19(1-2):25--31, 2001.

\bibitem{pace2013locally}
D.~F. Pace, S.~R. Aylward, and M.~Niethammer.
\newblock A locally adaptive regularization based on anisotropic diffusion for
  deformable image registration of sliding organs.
\newblock {\em IEEE transactions on medical imaging}, 32(11):2114--2126, 2013.

\bibitem{risser2013piecewise}
L.~Risser, F.-X. Vialard, H.~Y. Baluwala, and J.~A. Schnabel.
\newblock Piecewise-diffeomorphic image registration: Application to the motion
  estimation between {3D CT} lung images with sliding conditions.
\newblock {\em Medical image analysis}, 17(2):182--193, 2013.

\bibitem{risser2010simultaneous}
L.~Risser, F.-X. Vialard, R.~Wolz, D.~D. Holm, and D.~Rueckert.
\newblock Simultaneous fine and coarse diffeomorphic registration: application
  to atrophy measurement in {Alzheimer’s} disease.
\newblock In {\em MICCAI}, pages 610--617. Springer, 2010.

\bibitem{rohe2017svf}
M.-M. Roh{\'e}, M.~Datar, T.~Heimann, M.~Sermesant, and X.~Pennec.
\newblock {SVF-Net}: Learning deformable image registration using shape
  matching.
\newblock In {\em MICCAI}, pages 266--274. Springer, 2017.

\bibitem{rueckert1999nonrigid}
D.~Rueckert, L.~I. Sonoda, C.~Hayes, D.~L. Hill, M.~O. Leach, and D.~J. Hawkes.
\newblock Nonrigid registration using free-form deformations: application to
  breast {MR} images.
\newblock {\em TMI}, 18(8):712--721, 1999.

\bibitem{schmah2013left}
T.~Schmah, L.~Risser, and F.-X. Vialard.
\newblock Left-invariant metrics for diffeomorphic image registration with
  spatially-varying regularisation.
\newblock In {\em International Conference on Medical Image Computing and
  Computer-Assisted Intervention}, pages 203--210. Springer, 2013.

\bibitem{shen2002hammer}
D.~Shen and C.~Davatzikos.
\newblock {HAMMER:} hierarchical attribute matching mechanism for elastic
  registration.
\newblock {\em TMI}, 21(11):1421--1439, 2002.

\bibitem{shen2019networks}
Z.~Shen, X.~Han, Z.~Xu, and M.~Niethammer.
\newblock Networks for joint affine and non-parametric image registration.
\newblock {\em CVPR}, 2019.

\bibitem{simpson2015probabilistic}
I.~J. Simpson, M.~J. Cardoso, M.~Modat, D.~M. Cash, M.~W. Woolrich, J.~L.
  Andersson, J.~A. Schnabel, S.~Ourselin, A.~D.~N. Initiative, et~al.
\newblock Probabilistic non-linear registration with spatially adaptive
  regularisation.
\newblock {\em Medical image analysis}, 26(1):203--216, 2015.

\bibitem{singh2013vector}
N.~Singh, J.~Hinkle, S.~Joshi, and P.~T. Fletcher.
\newblock A vector momenta formulation of diffeomorphisms for improved geodesic
  regression and atlas construction.
\newblock In {\em 2013 IEEE 10th International Symposium on Biomedical
  Imaging}, pages 1219--1222. IEEE, 2013.

\bibitem{stefanescu2004grid}
R.~Stefanescu, X.~Pennec, and N.~Ayache.
\newblock Grid powered nonlinear image registration with locally adaptive
  regularization.
\newblock {\em Medical image analysis}, 8(3):325--342, 2004.

\bibitem{vercauteren2008symmetric}
T.~Vercauteren, X.~Pennec, A.~Perchant, and N.~Ayache.
\newblock Symmetric log-domain diffeomorphic registration: A demons-based
  approach.
\newblock In {\em MICCAI}, pages 754--761. Springer, 2008.

\bibitem{vercauteren2009diffeomorphic}
T.~Vercauteren, X.~Pennec, A.~Perchant, and N.~Ayache.
\newblock Diffeomorphic demons: Efficient non-parametric image registration.
\newblock {\em NeuroImage}, 45(1):S61--S72, 2009.

\bibitem{vialard2014spatially}
F.-X. Vialard and L.~Risser.
\newblock Spatially-varying metric learning for diffeomorphic image
  registration: A variational framework.
\newblock In {\em International Conference on Medical Image Computing and
  Computer-Assisted Intervention}, pages 227--234. Springer, 2014.

\bibitem{vialard2012diffeomorphic}
F.-X. Vialard, L.~Risser, D.~Rueckert, and C.~J. Cotter.
\newblock Diffeomorphic {3D} image registration via geodesic shooting using an
  efficient adjoint calculation.
\newblock {\em International Journal of Computer Vision}, 97(2):229--241, 2012.

\bibitem{wulff2015efficient}
J.~Wulff and M.~J. Black.
\newblock Efficient sparse-to-dense optical flow estimation using a learned
  basis and layers.
\newblock In {\em CVPR}, pages 120--130, 2015.

\bibitem{yang2016fast}
X.~Yang, R.~Kwitt, and M.~Niethammer.
\newblock Fast predictive image registration.
\newblock In {\em DLMIA}, pages 48--57. Springer, 2016.

\bibitem{yang2017quicksilver}
X.~Yang, R.~Kwitt, M.~Styner, and M.~Niethammer.
\newblock Quicksilver: Fast predictive image registration--a deep learning
  approach.
\newblock {\em NeuroImage}, 158:378--396, 2017.

\bibitem{younes2009evolutions}
L.~Younes, F.~Arrate, and M.~I. Miller.
\newblock Evolutions equations in computational anatomy.
\newblock {\em NeuroImage}, 45(1):S40--S50, 2009.

\end{thebibliography}

}
\newpage
\setcounter{page}{1}
\section{Supplementary Material}

     This supplementary material provides additional details illustrating the proposed approach. We start by deriving the form of the smoothing kernel for the RDMM model in Sec.~\ref{sec:kernel}. Based on this kernel form we can then detail the derivation of the RDMM optimality conditions in Sec.~\ref{sec:optimality_conditions}. In Sec.~\ref{sec:energy_conservation}, we prove that the regularization energy is conserved over time, which allows formulating our RDMM shooting strategy based on initial conditions only. Sec.~\ref{sec:mathematical_properties} details the good theoretical behavior of our model.
 Sec.~\ref{sec:initial_reg} describes the optimization/training strategy with regard to the the initial pre-weight regularization. Sec.~\ref{sec:unit} visualizes the inference process of the LDDMM/RDMM method. Sec.~\ref{sec:lp_derivation} analyzes the behavior of the OMT term. Lastly, Sec.~\ref{sec:experimental_settings} details the settings of our experiments.

\subsection{Variational Derivation of the Smoothing Kernel}
\label{sec:kernel}

The derivation of our RDMM model makes use of a smoothing kernel of the form $K=\sum_{i=0}^{N-1} w_i K_{\sigma_i}w_i$. This kernel form is a direct consequence of the definition of our variational definition of the smoothing kernel. 

Recall that similar to~\citep{niethammer2019_cvpr} we define
\begin{equation}
\|v \|^2_L := \inf \left\{ \sum_{i=0}^{N-1} \|\nu_i \|^2_{V_i}\,| \, v = \sum_{i=0}^{N-1} w_i \nu_i \right\}\,,
\end{equation}
for a given velocity field $v$. To compute an explicit form of the norm $\|v\|_L^2$ we need to solve the constrained optimization problem of this definition. Specifically, we introduce the vector-valued Lagrange multiplier $m$. Thus the Lagrangian, $\mathcal{L}$, becomes\footnote{We multiply the objective function by $\frac{1}{2}$ for convenience. This does not change the solution.}
\begin{align}
    \begin{split}
         \mathcal{L}(\{\nu_i\},m)  &=  \sum_{i=0}^{N-1} \frac 12\|\nu_i\|_{V_i}^2 - \langle m, w_i \nu_i -v \rangle\\
         & = \sum_{i=0}^{N-1} \frac 12 \langle L_i \nu_i,\nu_i \rangle - \langle m,  w_i\nu_i-v\rangle\,.
    \end{split}
\end{align}
The variation of the Lagrangian is
\begin{equation}
         \delta \mathcal{L}
(\{\nu_i\},m;\{\delta \nu_i\},\delta m) = \sum_{i=0}^{N-1}\langle L_i \nu_i,\delta \nu_i \rangle - \langle \delta m,  w_i\nu_i-v\rangle-\langle m,  w_i \delta \nu_i\rangle\,.
\end{equation}
By collecting all the terms, the optimality conditions (\ie, where the variation vanishes) are 
\begin{equation}
    L_i\nu_i=w_i m,~\forall i \quad\text{and}\quad v = \sum_{i=0}^{N-1} w_i\nu_i.
\end{equation}
Hence, we can write the norm $\|v(t) \|^2_L$ in the following form :
\begin{equation}
        \|v \|^2_L  =\sum_{i=0}^{N-1} \langle L_i \nu_i^*, \nu_i^* \rangle
        = \sum_{i=0}^{N-1} \langle w_i m, L_i^{-1} w_i m\rangle
        = \sum_{i=0}^{N-1} \langle m, w_i K_{\sigma_i} \zy{\star(w_i m)}\rangle \,.
\end{equation}
Consequentially, the associated kernel is $K = \sum_{i=0}^{N-1} w_i K_{\sigma_i} w_i$. Assuming the kernel can be written as a convolution, we can therefore express the velocity as:
\begin{equation}
v=\sum_{i=0}^{N-1} w_i \nu_i =  \sum_{i=0}^{N-1} w_{i} K_{\sigma_i} \star (w_{i}m)\,.
\end{equation}

\subsection{Optimality Conditions}
\label{sec:optimality_conditions}

In this section we derive the RDMM optimality conditions. Both for the image-based and the momentum-based cases. Recall that the overall registration energy of RDMM can be written as:
\begin{equation}\label{A_reg_energy}
E(v, I,\{h_i\}) =  \frac 12 \int_0^1 \|v(t) \|^2_L \ud t + \on{Sim}(I(1),I_1)\ +\on{Reg}(\{h_i(0)\})
\end{equation}
under the constraints\footnote{In this section, to simplify the notation, we denote the partial derivative $\partial_t$ by only the subscript $_t$.}
\begin{align}\label{A_reg_condition}
\begin{cases}
I_t + \langle \nabla I , v \rangle = 0,~I(0)=I_0,\\
(h_i)_{t} + \langle \nabla h_i , v \rangle = 0,h_i(0)=(h_i)_0,\\
w_i=G_{\sigma}\star h_i,\\\
\nu_i = K_{\sigma_i} \star(w_i m),\\
v = \sum_{i=0}^{N-1} w_i \nu_i \,.\\
\end{cases}
\end{align}

{\bf Proof of Thm.~\eqref{thm:rdmm_image_optimality}}\\
We compute the variations of the Lagrangian, $\mathcal{L}$ to the energy (i.e., where constraints are added via Lagrangian multipliers) with respect to $v$, $\lambda$, $I$, $\{h_i\}$ and $\{\gamma_i\}$:
\begin{align} \label{A3}
\begin{split}
 \delta \mathcal{L} =& \frac{\partial}{\partial \epsilon} \mathcal{L}\left.(v+\epsilon d v, I+\epsilon d I,\{h_i+\epsilon dh_i\}, \lambda+\epsilon d \lambda, \{\gamma_i+\epsilon d \gamma_i)\}\right|_{\epsilon=0} \\
 = &\int_0^1   \frac 12 \delta( \|v(t) \|^2_L)  -\left\langle d\lambda,I_t+ (DI) v\right\rangle
 -\left\langle\lambda, d I_{t}+(D d I) v+(D I) d v\right\rangle\\
 &-\sum_{i=0}^{N-1}\left\{\langle d\gamma_i, h_{it}+ (D h_i) v \right\rangle+ \left\langle\gamma_i, d h_{it}+(D d h_i) v+(D h_i) d v\rangle\right\} \ud t\\
 & +\left\langle\frac{\delta}{\delta I(1)}\on{Sim}(I(1),I_1),dI(1)\right\rangle + \sum_{i=0}^{N-1}\left\langle\frac{\delta}{\delta h_i(0)}\on{Reg}(\{h_i(0)\}),dh_i(0)\right\rangle\,.
 \end{split}
\end{align}

We use 
\begin{equation}\label{A4}
\int_{0}^{1}\left\langle\lambda, d I_{t}\right\rangle d t=\int_{0}^{1}\left\langle-\lambda_{t}, d I\right\rangle d t+\langle\lambda, d I\rangle_{0}^{1}\,.
\end{equation}
According to Green's theorem and assuming $v$ vanishes on the boundary, we get
\begin{equation}\label{A5}
\langle\lambda,(D d I) v\rangle=\langle- d i v(\lambda v), d I\rangle+\int_{\partial \Omega} d I \lambda v \cdot d S=\langle- d i v(\lambda v), d I\rangle\,.
\end{equation}
Similarly, we have
\begin{equation}\label{A6}
\int_{0}^{1}\left\langle\gamma_i, d h_{it}\right\rangle d t=\int_{0}^{1}\left\langle-\gamma_{it}, d h_i\right\rangle d t+\langle\gamma_i, d h_i\rangle_{0}^{1}\\
\end{equation}
\begin{equation}\label{A7}
\langle\gamma_i,(D d h_i) v\rangle=\langle- d i v(\gamma_i v), d h_i\rangle+\int_{\partial \Omega} d h_i \gamma_i v \cdot d S=\langle- d i v(\gamma_i v), d h_i\rangle \,.
\end{equation}
Now, Eq.~\eqref{A3} reads
\begin{align}
\begin{split}
 \delta E = &\int_0^1   \frac 12 \delta( \|v(t) \|^2_L)  -\left\langle d\lambda,I_t+ (D I ) v \right\rangle
 +\left\langle \lambda_t + div(\lambda v), dI\right\rangle\\
 &+\sum_{i=0}^{N-1}\left\{-\langle d\gamma_i, h_{it}+ (Dh_i) v \right\rangle + \left\langle\gamma_{it} + div(\gamma_i v), dh_i\rangle\right\}  -\left\langle \lambda \nabla I + \sum_{i=0}^{N-1}\gamma_i\nabla h_i, dv\right\rangle \ud t \\
 & -\langle \lambda, dI \rangle^1_0 - \sum_{i=0}^{N-1}\langle\gamma_i, dh_i\rangle^1_0\\
 &+\left\langle\frac{\delta}{\delta I(1)}\on{Sim}(I(1),I_1),dI(1)\right\rangle + \sum_{i=0}^{N-1}\left\langle\frac{\delta}{\delta h_i(0)}\on{Reg}(\{h_i(0)\}),dh_i(0)\right\rangle\,.
\end{split}
\end{align}

We first collect $dI(1)$ and $dh_i(0)$ to obtain the final condition on $\lambda$ and the initial condition on $\gamma$:
\begin{align}
\begin{cases}
-\lambda(1) +  \frac{\delta}{\delta I(1)}\on{Sim}(I(1),I_1) = 0,\\
\gamma_i(0) +\frac{\delta}{\delta h_i(0)}\on{Reg}(\{h_i(0)\}) = 0\,.
\end{cases}
\end{align}

Next, we work on $\int_0^1   \frac 12 \delta( \|v(t) \|^2_L) \ud t$.  Remember,  we have $
v=K \star m \stackrel{\mathrm{def.}}{=} \sum_{i=0}^{N-1} w_{i}\nu_i$, where $ \nu_i = K_{\sigma_i} \star(w_i m), \quad w_{i} \geq 0
$, thus

\begin{equation}\label{eq:decomp_dvK}
\int_0^1   \frac 12 \delta( \|v(t) \|^2_L)\ud t =\int_0^1  \frac 12\langle dm,v\rangle  + \frac 12
\langle m,\sum_{i=0}^{N-1} w_i d\nu_i + \nu_i dw_i  \rangle \ud t\,.
\end{equation}

Note that for radially symmetric kernels (such as Gaussian kernels) $K=\overline{K}$,  $\langle  K * a, b\rangle  = \langle \overline{K} * b\rangle $
\begin{align}
\begin{split}
\langle K * a, b\rangle &=\int_{x=-\infty}^{\infty}\left(\int_{y=-\infty}^{\infty} K(x-y) a(y)\right) b(x) d x \\ &=\int_{y=-\infty}^{\infty} a(y) \int_{x=-\infty}^{\infty} K(x-y) b(x) d x d y \\ &=\int_{y=-\infty}^{\infty} a(y) \int_{x=-\infty}^{\infty} \underbrace{\overline{K}(y-x)}_{\overline{K}(x) : K(-x)} b(x) d x d y \\ &=\int_{x=-\infty}^{\infty} a(x) \int_{y=-\infty}^{\infty} \overline{K}(x-y) b(y) d y d x \\ &=\langle a, \overline{K} * b\rangle\,.
\end{split}
\end{align}
Thus, we can get
\begin{align}\label{eq:decomp_dv_m}
\begin{split}
\frac 12 \langle \nu_i dw_i +  w_i d\nu_i , m \rangle=& \frac 12 \langle  dw_i K_{\sigma_i} \star (w_i m) + w_i K_{\sigma_i} \star (dw_i m + w_i dm),m\rangle \\
=&\frac 12 \langle m^TK_{\sigma_i}\star(w_i m), dw_i\rangle  + \frac 12\langle  w_i m , K_{\sigma_i} \star (dw_i m)\rangle  + \frac 12 \langle  w_i m, K_{\sigma_i} \star (w_i dm)\rangle \\
=&\frac 12  \langle m^T K_{\sigma_i}\star(w_i m), dw_i\rangle  + \frac 12\langle m^T K_{\sigma_i} \star  (w_i m) ,dw_i\rangle  + \frac 12\langle  w_i K_{\sigma_i} \star (w_i m), dm\rangle \\
=&  \langle G_{\sigma}\star(m^T\nu_i), dh_i\rangle  + \frac 12 \langle w_i \nu_i,dm\rangle \,.
\end{split}
\end{align}

Substituting Eq.~\eqref{eq:decomp_dv_m} into Eq.~\eqref{eq:decomp_dvK}, we get 
\begin{equation}
 \int_0^1   \frac 12 \delta( \|v(t) \|^2_L)\ud t  =\int_0^1  \langle dm,v\rangle  +  \sum_{i=0}^{N-1}\langle G_{\sigma}\star(m^T\nu_i), dh_i\rangle \ud t\,.
\end{equation}


Next, we decompose the $\langle\lambda \nabla I + \sum_{i=0}^{N-1} \gamma_i \nabla h_i, dv\rangle$ terms. We define the momentum,  $m = \lambda \nabla I + \sum_{i=0}^{N-1} \gamma_i \nabla h_i $.
\begin{flalign*}
&\langle \lambda \nabla I + \sum_{i=0}^{N-1} \gamma_i \nabla h_i, dv\rangle 
= \langle {m}, \sum_{i=0}^{N-1} dw_i K_{\sigma_i} \star(w_i m) + w_i K_{\sigma_i} \star ( dw_i m + w_i dm)\rangle \\
=&\sum_{i=0}^{N-1} \langle {m}^T K_{\sigma_i}\star(w_i m), dw_i\rangle  + \langle m^T K_{\sigma_i}*(w_i {m}),dw_i\rangle  + \langle w_i K_{\sigma_i} \star(w_i{m}), dm\rangle \\
= & \sum_{i=0}^{N-1} \langle G_{\sigma}\star[{m}^TK_{\sigma_i} \star(w_im) + m^T K_{\sigma_i}\star(w_i{m})], dh_i\rangle  + \langle w_i K_{\sigma_i} \star(w_i {m}),dm\rangle \\
= &\sum_{i=0}^{N-1} 2\langle G_{\sigma}\star(m^T\nu_i), dh_i\rangle  + \langle w_i \nu_i,dm\rangle \,.
\end{flalign*}

Now, we can collect the variation $dh_i$ for $h_i$ and $dm$ for $m$ and obtain the optimality conditions 
\begin{align}
- G_{\sigma}\star(m^T\nu_i) + \gamma_{it} + div(\gamma_i v) &=0,\\
 v- \sum_{i=0}^{N-1} w_i \nu_i &=0\,.
\end{align}


Finally, we get the optimality conditions for image-based RDMM derived from Eq.~\eqref{A_reg_energy} and Eq.~\eqref{A_reg_condition}:
\begin{align}
\begin{cases}
 I_t + \langle \nabla I , v \rangle = 0,~I(0)=I_0\,,\\
 h_{it} + \langle \nabla h_i , v \rangle = 0,~h_i(0)=(h_i)_0\,,\\
 \lambda_t + \on{div}(\lambda v) = 0\,,\\
  \gamma_{it} + \on{div}(\gamma_i v) = G_{\sigma}\, \star (m \cdot \nu_i)\, ,\\
  -\lambda(1) +  \frac{\delta}{\delta I(1)}\on{Sim}(I(1),I_1) = 0\, ,\\
\gamma_i(0) +\frac{\delta}{\delta h_i(0)}\on{Reg}(\{h_i(0)\}) = 0\, ,
\end{cases}
\label{eq:app_image_based_optimality}
\end{align}
where $\nu_i = K_{\sigma_i} \star (w_i m)$ and $m=\lambda \nabla I + \sum_{i=0}^{N-1} \gamma_i \nabla h_i$.\\

{\bf Proof of Thm.~\eqref{thm:rdmm_momentum_optimality}}
We now derive the optimality conditions for the momentum-based formulation of RDMM. We start by taking the time derivative of the momentum and obtain
\begin{eqnarray}
-m_t &=& -(\lambda \nabla I)_t - (\sum_{i=0}^{N-1}\nabla h_i\gamma_i)_t\\
&=& -\lambda_t \nabla I - \lambda \nabla I_t - \sum_{i=0}^{N-1} \{ \gamma_{it} \nabla h_i + \gamma_i \nabla(h_{it})\}\,.
\end{eqnarray}
By substituting the time derivatives $\lambda_t$, $I_t$, $\gamma_{it}$, and $h_{it}$ from Eq.~\eqref{eq:app_image_based_optimality} we obtain
\begin{equation}
-m_t = \text{div}(\lambda v)\nabla I + \lambda \nabla (\nabla I^T v) + \sum_{i=0}^{N-1} \left[ \text{div}(\gamma_i v) -G_{\sigma}\star(m^T \nu_i) \right] \nabla h_i
+\gamma_i \nabla(\nabla h_i^T v)\,.
\label{eq:app_m_opt_t_derivative}
\end{equation}
Using the following two relations,
\begin{equation}
    \text{div}(\lambda v) = \nabla \lambda^T v + \lambda \text{div}(v)\quad\text{and}\quad \nabla(\nabla I^Tv) = HIv + (Dv)^T\nabla I
\end{equation}
where $D$ denotes the Jacobian and $H$ the Hessian, we can rewrite Eq.~\eqref{eq:app_m_opt_t_derivative} as
\begin{eqnarray}
-m_t &=& ((\nabla \lambda)^T v + \lambda \text{div}(v))\nabla I + \lambda(HIv+(Dv)^T\nabla I)\\
&+& \sum_{i=0}^{N-1} \left [(\nabla \gamma_i)^T v +\gamma_i \text{div}(v) -G_{\sigma}\star(m^T \nu_i) \right] \nabla h_i + \gamma_i (Hh_i v+ (Dv)^T \nabla h_i\\
&=& (\lambda \nabla I + \sum_{i=0}^{N-1} \gamma_i \nabla h_i)\text{div}(v)+(Dv)^T[\lambda \nabla I + \sum_{i=0}^{N-1} \gamma_i \nabla h_i] +(\nabla \lambda^T v)\nabla I \\
&+& \lambda HIv + \sum_{i=0}^{N-1} \left[ (\nabla \gamma_i)^T v\right]\nabla h_i + \gamma_i H h_iv 
- G_{\sigma}\star(m^T \nu_i)\nabla h_i\,.
\end{eqnarray}
Noticing that
\begin{equation}
    D(\lambda \nabla I)v = \lambda HIv + \nabla\lambda^T v\nabla I
\end{equation}
we can write 
\begin{eqnarray}
    &~&(\nabla\lambda^T v)\nabla I + \lambda HIv + \sum_{i=0}^{N-1}((\nabla\gamma_i)^T v)\nabla h_i + \gamma_i Hh_i v\\
    &~~~~~& = D(\lambda\nabla I)v + \sum_{i=0}^{N-1} D(\gamma_i\nabla h_i)v = (Dm)v\,.
\end{eqnarray}
Finally, we get
\begin{equation}
-m_t = m\text{div}(v) + (Dv)^T m + (Dm)v - \sum_{i=0}^{N-1} G_{\sigma}\star(m^T \nu_i)\nabla h_i\,,
\end{equation}
which gives the result.

\subsection{Energy Conservation}
\label{sec:energy_conservation}

To formulate a shooting-based solution we would like to avoid integrating $\|v\|_L^2$ over time. We here show that this quantity is conserved. Hence, $\int_0^1 \|v\|_L^2~\ud t=\|v(0)\|_L^2$, which allows us to write our shooting equations only with respect to initial conditions subject to the momentum-based evolution equations of RDMM.

Recall that the energy is preserved by the EPDiff equation since it can alternatively be written as 
\begin{equation}
    \partial_t m + \on{ad}_{v}^*m = 0\,,
\end{equation}
where $\on{ad}^*$ is the adjoint of $\on{ad}_v w := \ud v(w) - \ud w(v)$.
It implies that 
\begin{align*}
  \frac{\ud }{\ud t}  \langle m,K \star m\rangle =- 2 \langle \on{ad}_{v}^*m,K\star m\rangle\, = \langle m,\on{ad}_v v\rangle = 0\,,
\end{align*}
since $\on{ad}_v v = 0$. In fact, there is more than conservation of the energy, since the momentum is actually advected along the flow. Now, formula \eqref{EqREPDiff} can be shortened as $\partial_t m + \on{ad}_{v}^*m = \sum_i G_\sigma \star (m^T \nu_i) \nabla h_i$ and it implies that, denoting the kernel $K(w_i)$ to shorten the notations,
\begin{align*}
  \frac{\ud }{\ud t}  \langle m,K(w_i) \star m\rangle &=- 2 \langle \on{ad}_{v}^*m,K(w_i)\star m\rangle + 2\langle \sum_i G_\sigma \star  (m^T \nu_i) \nabla h_i,v\rangle + 2\langle m,\sum_i \partial_t w_i \nu_i\rangle \\
  & = 2\langle \sum_i G_\sigma \star (m \nu_i) \nabla h_i,v\rangle + 2\langle m ,\sum_i (\partial_t w_i) \nu_i\rangle = 0
\end{align*}
since the first term vanishes as for the standard EPDiff equation and the two other terms cancel each other since $\partial_t w_i = -G_\sigma \star \nabla h_i \cdot v$ and $v = \sum_i G_{\sigma_i} \star (w_i m)$. Here we assumed the kernel to be symmetric in writing this equation but the result holds in general, the equations being simply modified with the transpose kernel.

\subsection{Mathematical properties}
\label{sec:mathematical_properties}

In this section, we prove that given $(\nu_i)_{i=0,\ldots,N-1}$, there exists a solution $\varphi(t)$ solving Equations \eqref{A_reg_energy} and \eqref{A_reg_condition} at least until a time $T>0$ which could be less than $1$. The notations $\| \cdot \|_{k,\infty}$ or $\| \cdot \|_{C^k}$ denote the sup norm of $C^k$ maps.

\begin{theorem}
Let $V_{N-1} \subset \ldots \subset V_0$ and suppose for every $\nu_k \in V_k$, $\| \nu_k \|_{V_k} \leq const \| \nu_k\|_{V_1} \leq const \| \nu_k \|_{2,\infty}$.
Given initial weights $(h_i(t = 0))_{i =0}^{N-1} \in L^2$ and time dependent vector fields $\nu_i(t) \in V_i$, there exists a unique solution $\varphi(t)$ to Equations \eqref{eq:rdmm_image_constraints} until time $1$.
\end{theorem}

\begin{proof}
The first step consists in proving that there exists a solution locally in time. To this end, the proof follows a fixed point argument on the space $C^0([0,1],\on{Diff}_{C^1}(\Omega))$, \ie  the space of continuous curves in $\on{Diff}_{C^1}(\Omega)$, for the map
\begin{equation}
T(\varphi) := \on{Fl}(v)
\end{equation}
where $v$ is defined as $v[\varphi] := \sum_{i = 0}^{N-1} G_\sigma \star h_i(\varphi^{-1}(t,y))  \nu_i(t,\varphi(t,x))$. 
The existence of the flow associated with $v[\varphi]$ is ensured by standard arguments provided that the Lipschitz constant of $v[\varphi]$ is bounded. It is the case since $G_\sigma(x,y) \star h_i(\varphi^{-1}(t,y))$ has a Lipschitz constant bounded by $ \sup_{x \in \Omega }| \partial_1 G_{\sigma}(x,y)|$ since $| h_i(\varphi^{-1}(t,y)) | \leq 1$. This gives $\| v[\varphi] \|_{1,\infty} \leq  \sum_{i = 0}^{N-1} M \| \nu_i \|_{1,\infty}$ and the constant $M$ does not depend on $\varphi$. A similar inequality holds for the sup norm on the  derivatives up to order $k$ provided each space $V_i$ continuously embeds in $C^k$.
\par
One has the inequality
\begin{equation}\label{EqFirstEstimate}
\| T(\varphi)(t) \|_{2,\infty} \leq e^{\int_0^t \| v[\varphi] \|_{2,\infty} \ud s}
\end{equation}
and therefore, $\| T(\varphi)(t) \|_{2,\infty}$ is bounded a priori by a positive constant which does not depend on $\varphi$.
Using Gronwall's lemma \eqref{ThmGronwall}, one has also
\begin{equation}
\| T(\varphi)(t) - T(\psi)(t) \|_{1,\infty} \leq \sqrt{t} \| v[\varphi] - w[\psi] \|_{L^2([0,t],C^1)} e^{\int_0^t (1 + \| \varphi \|_{1,\infty}) \| v \|_{C^2} \ud s}\,.
\end{equation}
Moreover, by a change of variable $y=\varphi(t,x)$ we have
\begin{equation}
G_\sigma \star h_i(\varphi^{-1}(t,y)) = G_\sigma(x,\varphi(t,y)) \star (\on{Jac}(\varphi(t,y))h_i(y))\,
\end{equation}
and therefore
\begin{equation}
\| G_\sigma \star h_i(\varphi^{-1}(t,y))  \nu_i - G_\sigma \star h_i(\psi^{-1}(t,y))  \nu_i \|_{C^1}  \leq M' \| \varphi - \psi \|_{C^1} \| \nu_i \|_{C^1}\,.
\end{equation}
Thus, we deduce the inequality
\begin{equation}
\| v[\varphi] - w[\psi] \|_{L^2([0,t],C^1)} \leq M \sup_{s \in [0,t]} \| \psi(s) - \varphi(s) \|_{C^1}\,,
\end{equation}
therefore, the map $T$ is a contraction for a time $T$ small enough. 
Using Equation \eqref{EqFirstEstimate}, it is easily seen that this existence can be applied on $[T,2T]$ and iterating this argument shows existence until time $t=1$.
\end{proof}

\begin{theorem}
The variational problem \eqref{A_reg_energy} under the constraints of Equations \eqref{A_reg_condition} has a solution.
\end{theorem}

\begin{proof}
The direct method of calculus of variations can be applied here, see Sect.~\ref{sec:optimality_conditions}. The sum of squared norms are lower semicontinuous; The penalty term as well as the constraints are weakly closed for the weak convergence on $(\nu_i)$.
\end{proof}

\begin{lemma}
Let $u,v \in L^2([0,1],C^2)$ and let $\varphi,\psi$ be their associated flows. The following estimates hold,
\begin{equation}
\| \varphi(t) \|_{C^2} \leq e^{\int_0^t \| v(s) \|_{C^2} \ud s}\,,
\end{equation}
and 
\begin{equation}
\| \varphi(t) - \psi(t) \|_{C^1} \leq \sqrt{t}M \| u - v \|_{L^2([0,t],C^1)} e^{\int_0^t(1 + \| \varphi \|_{1,\infty}) \| v \|_{C^2} \ud s}\,.
\end{equation}
where $M$ is a constant that bounds $\| \varphi \|_{1,\infty}$.
\end{lemma}

\begin{proof}
Use Gronwall's lemma \eqref{ThmGronwall} recalled below on the following inequality coming from the flow equation
\begin{align}
\| \varphi(t) - \psi(t) \|_{1,\infty}& \leq \int_0^t \| \ud u \circ \varphi(t) \cdot \ud \varphi(t)- \ud w \circ \psi(t)\cdot \ud \psi(t) \|_{0,\infty}  \ud s\\
& \leq \int_0^t \| u-v \|_{1,\infty} \|\varphi \|_{1,\infty}\! +\! \| v \|_{2,\infty} \|\varphi \|_{1,\infty} \| \varphi - \psi \|_{0,\infty} \!+\! \| dv \|_{1,\infty} \| \varphi - \psi \|_{1,\infty}\ud s\\
& \leq \int_0^t \| u-v \|_{1,\infty} \|\varphi \|_{1,\infty}  + (1 +  \|\varphi \|_{1,\infty})\| v \|_{2,\infty} \| \varphi - \psi \|_{1,\infty}\ud s\,.
\end{align}
\end{proof}

Recall that Gronwall's lemma is 
\begin{lemma}\label{ThmGronwall}
Let $r$ be a nonnegative function on $\R$ such that 
\begin{equation}
r(t) \leq c(t) +\left| \int_0^t \alpha(s) r(s) \ud s \right|
\end{equation}
for given positive functions $\alpha$ and $c$.
Then, 
\begin{equation}
r(t) \leq c(t) + \left| \int_0^t \alpha(s) c(s) e^{|\int_0^t \alpha(s) \ud s|} \ud s \right| \,,
\end{equation}
and if $c$ is a constant, a further simplified formula is
\begin{equation}
r(t) \leq c e^{|\int_0^t \alpha(s) \ud s|}\,.
\end{equation}
\end{lemma}

\subsection {Initial Pre-weight Regularization}\label{sec:initial_reg}
The initial regularization term ${Reg}(\{h_i(0)\})$ determines the behavior of the initial regularizer. In our experiments, 
\begin{equation}
    {Reg}(\{h_i(0)\},T) =\lambda_{\mathrm{OMT}}(T) {OMT}(\{h_i(0)\}) + \lambda_{\mathrm{Range}}(T){Range}(\{h_i(0)\})\, ,
\end{equation}
where $\lambda_{\mathrm{OMT}}$ and $\lambda_{\mathrm{Range}}$ are scale factors; T refers to the iteration/epoch. Specifically,
\begin{equation}
    {OMT}=\left|\log \frac{\sigma_{N-1}}{\sigma_{0}}\right|^{-s} \sum_{i=0}^{N-1} w_{i}\left|\log \frac{\sigma_{N-1}}{\sigma_{i}}\right|^{s}
\end{equation}
where $s$ is the chosen power and
\begin{equation}
    {Range}= \|G_{\sigma}\star(h(0))-w_0\|^2_2
\end{equation}
where $w_0$ is the pre-defined initial weight. The range loss penalizes differences between the initial weight, $w(0)$, from the pre-defined one, $w_0$.

At the beginning of the optimization/training, it is difficult to jointly optimize over the momentum and the pre-weights. Hence, we constrain the pre-weights by introducing the Range loss that penalizes the difference between the optimized and pre-defined pre-weights. Besides, as we prefer well-regularized (\ie, smooth) transformation, we use the OMT loss to penalize weight assignments to Gaussians with small standard deviations. To solve the original model, the influence of the range penalty needs to diminish while the influence of the OMT term need to increase during training. In practice, we introduce epoch-dependent decay factors:
 \begin{equation}
 \lambda_T = \frac{K}{K+e^{T/K}}\quad,
  \lambda_{\mathrm{Range}} := C_{\mathrm{Range}}\lambda_T ,\quad
  \lambda_{\mathrm{OMT}} =C_{\mathrm{OMT}}(1 -\lambda_T),\quad
\label{eq:regularization_penalty_weights}
\end{equation}
where $C_{\mathrm{Range}}$ and $C_{\mathrm{OMT}}$ are pre-defined constants, $K$ controls the decay rate, and $T$ indicates the iteration/epoch.


\subsection{LDDMM/RDMM Unit}
\label{sec:unit}

Fig.~\ref{fig:pipeline} shows the flow charts for LDDMM and RDMM. \zy{Additionally, for RDMM with a pre-defined regularizer, we define $\{h_i(0)\}$ in the source image space, as foreground and background are easier to specify there. For RDMM with an optimized/learnt regularizer we define $\{h_i(0)\}$ in the pre-aligned image space, since the goal is to find the optimal initial conditions that determine the geodesic path based on the RDMM shooting equations; specifically, we take $\varphi^{-1}(0)=id$ as the input, and the final output composes the initial map and the transformation map, $\varphi^{-1}(1)$.}

\begin{figure}
\includegraphics[width=1\textwidth]{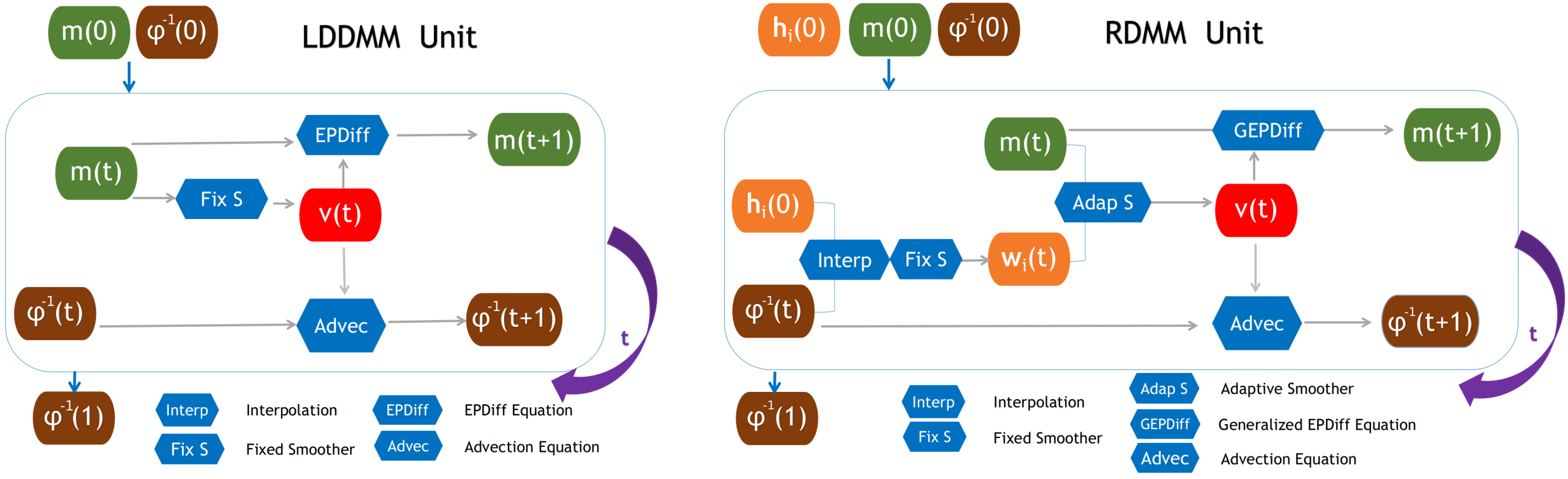}
\caption{Flow chart of LDDMM (left) and our RDMM model (right). LDDMM solves EPDiff and advects the transformation map, whereas in RDMM a modified EPDiff equation is solved combined with an advection of the transformation map and the pre-weights for the regularizer. Note that transformation map $\varphi^{-1}$ and the pre-weights $\{h_i\}$ are both advected according to the velocity field. Hence, instead of computing the advection equation, we update $\{h_i(t)\}$ by interpolating the initial pre-weights $\{h_i(0)\}$ via the current transformation map $\varphi^{-1}(t)$, which is more computationally efficient and avoids numerical dissipation.
}
\label{fig:pipeline}
\end{figure}

\subsection{Analysis of the OMT term}
\label{sec:lp_derivation}

This section illustrates the behavior of the OMT term to obtain regular solutions. To understand the behavior of the OMT term, we do some simple analysis. We assume
\begin{equation}
  0<\sigma_0<\sigma_1<\cdots<\sigma_{N-1}
\end{equation}
where $\sigma_i$ are the standard deviations of the Gaussians. In general, we desire
\begin{equation}
  c_0>c_1>\cdots>c_{N-1},
\end{equation}
where $c_i$ are the associated costs of assigning a weight to $i$-th Gaussian. That is, we assume that it gets progressively cheaper to assign to Gaussians with larger standard deviation. However, we do not assume that this is the case in the following derivations. Our OMT penalty term is then of the form
\begin{equation}
  f(w)=\sum_{i=0}^{N-1} c_i w_i = c^\top w
\end{equation}
with constraints
\begin{equation}
  \sum_{i=0}^{N-1} w_i -1 = 1^\top  w -1 = 0 \quad\text{and}\quad w_i\geq 0\enspace. 
\end{equation}
To study this term, assume that we have a given target standard deviation $\hat{\sigma}$ that we wish to explain via a multi-Gaussian. This results in the constraint
\begin{equation}
  \sum_{i=0}^{N-1} \sigma_i^2 w_i -\hat{\sigma}^2 = v^\top w -\hat{\sigma}^2 = 0\enspace.
\end{equation}
This optimization problem is linear in the multi-Gaussian weights, $w$, and consequentially constitutes the following linear program
\begin{equation}
  \underset{w}{\min}~f(w)\quad\text{s.t.} \quad
  \begin{cases}
    1^\top w-1&=0, \\
    v^\top w-\hat{\sigma}^2&=0,\\
    w_i&\geq 0\enspace.
  \end{cases}
\end{equation}
The Lagrangian of this problem is
\begin{equation}
  L(w,\lambda,\gamma_1,\gamma_2) = f(w) -\lambda^\top w-\gamma_1(1^\top w-1)-\gamma_2(v^\top w-\hat{\sigma}^2)\enspace,
\end{equation}
which results following KKT conditions~\cite{nocedal2006numerical}
\begin{eqnarray}
  c-\lambda-\gamma_1 1 - \gamma_2 v &=& 0, \label{eq:KKT_beg}\\ 
  1^\top  w -1 &=& 0,\\
  v^\top  w -\hat{\sigma}^2 &=& 0,\\
  w &\geq& 0,\\
  \lambda &\geq& 0, \\
  \lambda_iw_i&=&0,~\forall~i\enspace.\label{eq:KKT_end}
\end{eqnarray}

\subsubsection{Solution on a simplex edge}

Assume a solution candidate for the KKT conditions (Eqs.~\eqref{eq:KKT_beg}-\eqref{eq:KKT_end}) that only has two zero weights
\begin{equation}
  w_k>0,~w_l>0,~w_i=0~\forall i\notin \{k,l\},~\sigma_k<\sigma_l\enspace.
\end{equation}
Then, we know
\begin{equation}
  w_k + w_l =1,\quad \sigma_k^2 w_k + \sigma_l^2 w_l = \hat{\sigma}^2,\quad \lambda_k=\lambda_l=0\enspace.\label{eq:edge_conditions}
\end{equation}
Using Eq.~\eqref{eq:edge_conditions}, we can directly solve for $w_k$ and $w_l$ and obtain
\begin{equation}
  w_k = \frac{\sigma_l^2-\hat{\sigma}^2}{\sigma_l^2-\sigma_k^2},\quad w_l =\frac{\hat{\sigma}^2-\sigma_k^2}{\sigma_l^2-\sigma_k^2}\enspace.
\end{equation}
Note that these weights are independent of the costs $c_i$\footnote{We will show in the remainder of this section that for costs defined via a convex function, $g$, $c=g(\sigma^2)$, this is indeed a solution of the KKT equations. This is consistent with known optimal mass transport theory, where for convex costs the optimal mass transport in 1D is a monotone rearrangement~\cite{delon2012}.}. 

Both weights are required to be non-negative, \emph{i.e.}  $w_k\geq 0,w_l\geq 0$. From this condition, we obtain
\begin{eqnarray}
  w_k&=&\frac{\sigma_l^2-\hat{\sigma}^2}{\sigma_l^2-\sigma_k^2}\geq 0 \\
  &\Leftrightarrow& \sigma_l^2-\hat{\sigma}^2\geq 0 \\
  &\Leftrightarrow& \hat{\sigma}^2\leq \sigma_l^2
\end{eqnarray}
and
\begin{eqnarray}
  w_l&=&\frac{\hat{\sigma}^2-\sigma_k^2}{\sigma_l^2-\sigma_k^2}\geq 0 \\
  &\Leftrightarrow& \hat{\sigma}^2-\sigma_k^2\geq 0 \\
  &\Leftrightarrow& \sigma_k^2\leq \hat{\sigma}^2
\end{eqnarray}
and finally that the desired variance needs to be between the variances of $k$ and $l$, \ie, $\sigma_k^2\leq\hat{\sigma}^2\leq\sigma_l^2$. Since $\lambda_k=\lambda_l=0$, we further get from Eq.~\eqref{eq:KKT_beg} that
\begin{eqnarray}
  c_k -\gamma_1 -\gamma_2 \sigma_k^2 &=& 0\enspace,\\
  c_l -\gamma_1 -\gamma_2 \sigma_l^2 &=& 0\enspace,
\end{eqnarray}
which we can solve for $\gamma_1$ and $\gamma_2$ to obtain
\begin{equation}
  \gamma_1 = \frac{c_k\sigma_l^2-c_l\sigma_k^2}{\sigma_l^2-\sigma_k^2},\quad \gamma_2 = \frac{c_l-c_k}{\sigma_l^2-\sigma_k^2}\enspace.
\end{equation}
For arbitrary $i\notin\{l,k\}$ the Lagrangian multipliers are then
\begin{eqnarray}
  \lambda_i &=& -\gamma_1 -\gamma_2 \sigma_i^2 + c_i, \\
  &=& \frac{\sigma_l^2(c_i-c_k)+\sigma_k^2(c_l-c_i)+\sigma_i^2(c_k-c_l)}{\sigma_l^2-\sigma_k^2}, \\
  &=& \frac{c_i(\sigma_l^2-\sigma_k^2)+c_k(\sigma_i^2-\sigma_l^2)+c_l(\sigma_k^2-\sigma_i^2)}{\sigma_l^2-\sigma_k^2}.
\end{eqnarray}
Since $\lambda_i\geq 0$ we get
\begin{equation}
  c_i\geq c_k\frac{\sigma_l^2-\sigma_i^2}{\sigma_l^2-\sigma_k^2} + c_l\frac{\sigma_i^2-\sigma_k^2}{\sigma_l^2-\sigma_k^2} = g(\sigma_i^2).
\end{equation}
As $g(\sigma_k^2)=c_k$ and $g(\sigma_l^2)=c_l$, this is simply a line that passes through the points $(\sigma_k^2,c_k)$ and $(\sigma_l^2,c_l)$ and this condition states that for a solution candidate edge $(k,l)$ the costs for all $i\notin\{k,l\}$ are on or above this line. If the costs are defined via a function $c=h(\sigma^2)$, then, if $h$ is a convex function (and remembering the condition $\sigma_k^2\leq \hat{\sigma}^2\leq \sigma_l^2$), the optimal solution of this linear program will be on the edge $(k_*,l_*)$ most tightly bracketing $\hat{\sigma}^2$, \ie,
\begin{equation}
  \sigma_{k_*}^2\leq \hat{\sigma}^2\leq \sigma_{l_*}^2,\quad\text{s.t.}\quad k_* = \max_i \{i:\sigma_i^2\leq \hat{\sigma}^2\},~l_* = \min_i \{i:\hat{\sigma}^2\leq \sigma_i^2\}\enspace.
\end{equation}

\subsubsection{Solution on a simplex vertex}

Assume that $\hat{\sigma}^2=\sigma_j^2$, \ie, the desired variance coincides with the variances of one of the Gaussians. Then $w_j=1$ and $w_i=0,~\forall i\neq j$. Furthermore, we have $\lambda_j=0$ from which follows
\begin{equation}
  \gamma_1 = c_j-\gamma_2 \sigma_j^2\enspace.\label{eq:sv_gamma_1}
\end{equation}
For the remaining $N-1$ variables $i\neq j$, it needs to hold that
\begin{equation}
  c_i-\lambda_i-\gamma_1-\gamma_2\sigma_i^2=0\enspace.
\end{equation}
Substituting Eq.~\eqref{eq:sv_gamma_1}, we can solve for $\lambda_i$ and obtain
\begin{equation}
  \lambda_i = c_i-c_j + \gamma_2(\sigma_j^2-\sigma_i^2)\enspace.
\end{equation}
As all the Lagrangian multipliers, $\lambda$, are (by the KKT conditions) required to be non-negative, we obtain the condition
\begin{equation}
  c_i-c_j+\gamma_2(\sigma_j^2-\sigma_i^2)\geq 0\enspace.
\end{equation}
We can distinguish two conditions
\begin{eqnarray}
  \gamma_2 &\geq& \frac{c_j-c_i}{\sigma_j^2-\sigma_i^2},~\text{for}~i<j\enspace,\\
  \gamma_2 &\leq& \frac{c_k-c_j}{\sigma_k^2-\sigma_j^2},~\text{for}~j<k\enspace.
\end{eqnarray}
Hence, it needs to hold that
\begin{equation}
   \frac{c_j-c_i}{\sigma_j^2-\sigma_i^2} \leq \frac{c_k-c_j}{\sigma_k^2-\sigma_j^2},~\forall i<j<k\enspace.
\end{equation}
Since $\sigma_i^2<\sigma_j^2<\sigma_k^2$, we can multiply by $(\sigma_j^2-\sigma_i^2)(\sigma_k^2-\sigma_j^2)$ and obtain
\begin{equation}
  (c_j-c_i)(\sigma_k^2-\sigma_j^2)\leq (c_k-c_j)(\sigma_j^2-\sigma_i^2)\enspace.
\end{equation}
Solving for $c_j$, we get
\begin{equation}
  c_j \leq c_i\frac{\sigma_k^2-\sigma_j^2}{\sigma_k^2-\sigma_i^2} +  c_k\frac{\sigma_j^2-\sigma_i^2}{\sigma_k^2-\sigma_i^2}\enspace.
\end{equation}
Now, for a vertex $j$ to be a solution, this condition needs to be true for all $i<j<k$. Graphically, we can take any two points $(\sigma_i^2,c_i)$ and $(\sigma_k^2,c_k)$ and draw a line between them. If $(\sigma_k^2,_k)$ is below these lines for all pairs $(i,k)$ such that $i<k<j$, then there is a vertex solution, otherwise the solution is on an edge (for non-degenerate $c_i$). This condition is always fulfilled for convex functions $c=h(\sigma^2)$.

\subsubsection{Solution on a simplex face}

We can also ask if it is possible that a solution is on a face of the simplex. Consider the case of three non-zero weights $w_k>0,~w_l>0,~w_m>0$. In this case, we have
\begin{eqnarray}
  w_k + w_l + w_m &=& 1\enspace,\\
  \lambda_k = \lambda_l = \lambda_m &=& 0\enspace,\\
  \sigma_k^2 w_k + \sigma_l^2 w_l + \sigma_m^2 w_m &=&\hat{\sigma}^2\enspace.
\end{eqnarray}
Furthermore, to fulfill the KKT conditions, the following equation system needs to hold (and similarly for more than three non-zero weights):
\begin{equation}
  \begin{pmatrix}
    c_k \\
    c_l \\
    c_m
  \end{pmatrix} =
  \begin{pmatrix}
    1 & \sigma_k^2 \\
    1 & \sigma_l^2 \\
    1 & \sigma_m^2 
  \end{pmatrix}
  \begin{pmatrix}
    \gamma_1 \\
    \gamma_2
  \end{pmatrix}.
\end{equation}
For more than two non-zero weights this is an overdetermined system. Hence, a solution can only be on a face if this system has a solution. In that case, this means that
\begin{equation}
  c_m = c_l \frac{\sigma_k^2-\sigma_m^2}{\sigma_k^2-\sigma_l^2} + c_k \frac{\sigma_m^2-\sigma_l^2}{\sigma_k^2-\sigma_l^2}\enspace.
\end{equation}
In other words, this is only possible if $c=h(\sigma^2)$ is at least piecewise linear. Specifically, if $c=h(\sigma^2)$ is a {\it strictly} convex function, a solution can only be found on simplex edges or vertices based on the conditions in the two previous sections.

\subsubsection{Summary}
\label{subsec:lp_summary}

In summary, to assure that a solution exists either 1) based on the two Gaussians closest to the desired variance $\hat{\sigma}^2$ (\ie, a simplex edge) or 2) selecting exactly one of the Gaussians (a simplex vertex), it is desirable to pick a penalty function based on costs from a strictly convex function $c=h(\sigma^2)$. Based on our choice for the OMT penalty:
\begin{equation}
  c=h(x)=\frac{1}{2^r}\left(\log\frac{\sigma_{N-1}^2}{x}\right)^r\enspace,
\end{equation}
we get
\begin{equation}
  \frac{d^2h(x)}{dx^2}=\frac{\frac{r}{2^r}\log^{r-2}\left(\frac{\sigma_{N-1}^2}{x}\right)\left(\log\frac{\sigma_{N-1}^2}{x}+r-1\right)}{x^2}.
\end{equation}
For $0<x<\sigma_{N-1}^2$ and $r\geq 1$, the second derivative is positive and hence $h(x)$ is strictly convex.

\subsection{Experimental settings}
\label{sec:experimental_settings}

For all experiments, we normalize the intensities of each image such that the $0.1$th percentile and the $99.9$th percentile are mapped to ${0,1}$ and clamp values that are smaller to $0$ and larger to $1$ to avoid outliers. We also assume that spatial coordinates of images are in $[0, 1]^d$, where $d$ is the spatial dimension. This makes the interpretation of the standard deviations of the regularizers straightforward.

{\bf Non-parametric family} For numerical optimization solutions, we use three image scales $\{ 0.25, 0.5$ and $1.0\}$. We use L-BGFS as the optimizer. For the deep learning models, we train the multi-step affine network first and then train the non-parametric network with the affine network fixed. For all methods, we use a multi-kernel Gaussian regularizer with standard deviations $\sigma_i=\{ 0.05, 0.1, 0.15, 0.2, 0.25\}$. For both vSVF and LDDMM, we use fixed corresponding weights $w^2_0=\{0.067, 0.133, 0.2,0.267,0.333\}$, which is also set as the initial value for the range loss in RDMM.

{\bf Baseline methods} 

For the numerical optimization approaches, we compare with three public registration tools: NiftyReg~\cite{ourselin2001reconstructing,modat2014global,rueckert1999nonrigid,modat2010fast}, SyN~\cite{avants2009advanced,avants2008symmetric} and Demons~\cite{vercauteren2009diffeomorphic,vercauteren2008symmetric}. Besides, we also compare with two recent deep-learning approaches: VoxelMorph \cite{dalca2018unsupervised} and AVSM (vSVF-net)\cite{shen2019networks}. For a fair comparison, we take the same experimental settings as in \cite{shen2019networks}.

{\bf Registration with a pre-defined regularizer}

For the synthetic registration experiments, we use a multi-kernel Gaussian regularizer with standard deviations $\sigma_i= \{0.03, 0.06, 0.09, 0.3\}$ with initial pre-weights $h_0^2=\{0.2,0.5,0.3,0.0\}$ (fixed during the optimization) for the foreground (the dark blue region) and $h_0^2=\{0,0,0,1\}$ for the background (the cyan region). The standard deviation for $G_\sigma$, to smooth the pre-weights, is set to $0.02$. For each image scale, we compute 60 registration iterations. \\
For the lung registration we use $\sigma_i= \{0.04, 0.06, 0.08,0.2\}$ for the multi-Gaussian regularizer with initial pre-weights $h_0^2=\{0.1, 0.4, 0.5, 0\}$ for the foreground (\ie the lung) and $h_0^2=\{0,0,0,1\}$ everywhere outside the lung. We set $\sigma$ in $G_\sigma$ to $0.05$.  For each image scale, we compute 60 registration iterations.

{\bf Registration with an optimized regularizer}

For the synthetic registration experiment, we use $\sigma_i= \{0.02, 0.04, 0.06,0.08\}$ for the multi-Gaussian regularizer and the initial pre-weights $h_0^2=\{0.1, 0.3, 0.3, 0.3\}$. These initial value are set at the beginning of the optimization. They are the same for vSVF, LDDMM and RDMM. $C_{\mathrm{Range}}$ and $C_{\mathrm{OMT}}$ are set to 10 and 0.05 respectively; K is set to 10; $\sigma$ in $G_\sigma$ is set to $0.05$. For image scale \{0.25, 0.5 and 1.0\}, we  compute \{100, 100 and 400\} iterations, respectively.\\
For the knee MRI registration of the OAI dataset, we use $\sigma_i= \{0.05, 0.1, 0.15, 0.2, 0.25\}$ for the multi-Gaussian regularizer with the initial pre-weights $h_0^2=\{0.067, 0.133, 0.2, 0.267, 0.333\}$ for all non-parametric registration models.  $C_{\mathrm{Range}}$ and $C_{\mathrm{OMT}}$ are set to 1 and 0.25 respectively; K is set to 6; $\sigma$ in $G_\sigma$ is set to $0.06$. For each image scale, we compute 60 registration iterations.

{\bf Registration with a learnt regularizer} 

For the deep learning approaches, the settings are the same as for numerical optimization, as described above. Besides, we include an additional inverse consistency loss, with the scale factor set to 1e-4, for vSVF, LDDMM and RDMM to regularize the deformation.

\end{document}